\documentclass{article}

\usepackage[preprint]{neurips_2026}

\usepackage[utf8]{inputenc}
\usepackage[T1]{fontenc}
\usepackage{titletoc}
\usepackage{hyperref}
\usepackage{url}
\usepackage{booktabs}
\usepackage{amsfonts}
\usepackage{amssymb}
\usepackage{bm}
\usepackage{nicefrac}
\usepackage{microtype}
\usepackage{xcolor}
\usepackage{algorithm}
\usepackage{algpseudocode}
\usepackage{enumitem}
\usepackage{wrapfig}

\usepackage{conference}

\title{Signs Beat Floats: Low-Rank Double-Binary Adaptation for On-Device Fine-Tuning}

\author{%
Yoshihiko Fujisawa\thanks{Equal contribution.}\\
Fujitsu Limited,\\
Institute of Science Tokyo \\
\texttt{fujisawa.y.5bdc@m.isct.ac.jp} \\
\And 
Yuma Ichikawa\thanks{Equal contribution.} \\
Fujitsu Limited, \\
RIKEN Center for AIP \\
\texttt{ichikawa.yuma@fujitsu.com}
\And
Yudai Fujimoto \\
Fujitsu Limited,\\
Institute of Science Tokyo \\
\texttt{fujimoto.y.7de0@m.isct.ac.jp}
\And
Akira Sakai \\
Fujitsu Limited, \\
Tokai University \\
\texttt{akira.sakai@fujitsu.com}
\And
Katsuki Fujisawa \\
Institute of Science Tokyo \\
\texttt{fujisawa.k.2110@m.isct.ac.jp}
}

\begin{document}

\maketitle

\begin{abstract}
On-device adaptation of large language models commonly keeps a quantized base model frozen while training and deploying a small, task-specific LoRA adapter. In the unmerged adapter-mode setting, however, the adapter is more than a compact storage module; it introduces an additional dense floating-point branch, maintains a trainable state for local updates, and acts as a unit of communication and hot-swapping.
We introduce LoRDBA, a LoRA-compatible adapter that replaces both low-rank factors with binary sign carriers while representing magnitudes through lightweight, channel-wise scales, converting the dense adapter branch into two sign-accumulation matrix multiplications interleaved with channel-wise scaling. A finite-sample analysis shows that reconstruction quality is governed by the residual-to-magnitude ratio of the original LoRA factors. In adapter-mode experiments, LoRDBA outperforms low-bit baselines at matched model sizes while matching fp16 LoRA quality in selected regimes. The unmerged adapter incurs at most $8\%$ prefill latency overhead at matched rank $r{=}16$ despite an over $10\times$ reduction in adapter footprint, with moderate training memory overhead of approximately $1.6\times$ that of fp16 LoRA.

\end{abstract}

\section{Introduction}\label{sec:intro}

Large language models are increasingly adapted using parameter-efficient updates rather than full fine-tuning. LoRA enables this paradigm by freezing pretrained weights and injecting low-rank adapter matrices into Transformer projections, reducing the number of trainable parameters~\citep{hu2022lora}. QLoRA extends this approach to low-bit settings by backpropagating through a frozen quantized base model into LoRA adapters~\citep{dettmers2023qlora}. This base-plus-adapter design is attractive for on-device personalization, privacy-preserving adaptation, and communication-efficient update distribution; a compact low-bit base can remain fixed while task-specific adapters are trained, exchanged, or selected. Recent multi-adapter serving systems, including S-LoRA and Punica, demonstrate that practical deployment often treats adapters as swappable runtime objects rather than as weights that are always merged into the base model~\citep{sheng2024slora,chen2024punica}.

In practical low-bit deployments, merging a trained adapter into a quantized base requires materializing full-precision weights or requantizing per task, forfeiting hot-swap capability; it is therefore preferable to serve the adapter as an unmerged side branch~\citep{frantar2022gptq,lin2024awq}. This unmerged branch creates a distinct bottleneck: the base projection runs on fused low-bit kernels while the adapter uses dense fp16 factors. An adapter compressor for this regime must preserve accuracy, reduce unmerged computation cost, and avoid inflating the local training state.

\begin{wrapfigure}{r}{0.42\textwidth}
    \centering
    \vspace{-1.4em}
    \includegraphics[width=0.42\textwidth]{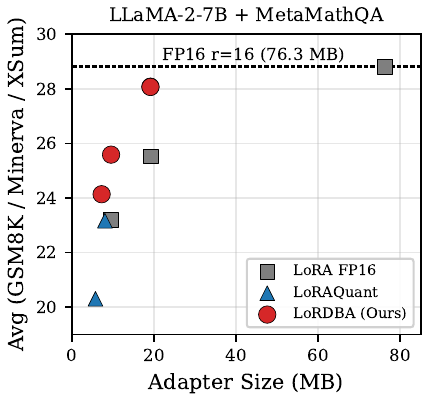}
    \vspace{-1.4em}
    \caption{\small LoRDBA (red) Pareto-dominates FP16 LoRA and LoRAQuant at matched adapter sizes on \textsc{LLaMA-2-7B}; cf.\ Table~\ref{tab:7b-size-ranking}.}
    \label{fig:intro-preview}
    \vspace{-1.0em}
\end{wrapfigure}

\begin{figure}[tb]
    \centering
    \includegraphics[width=\linewidth]{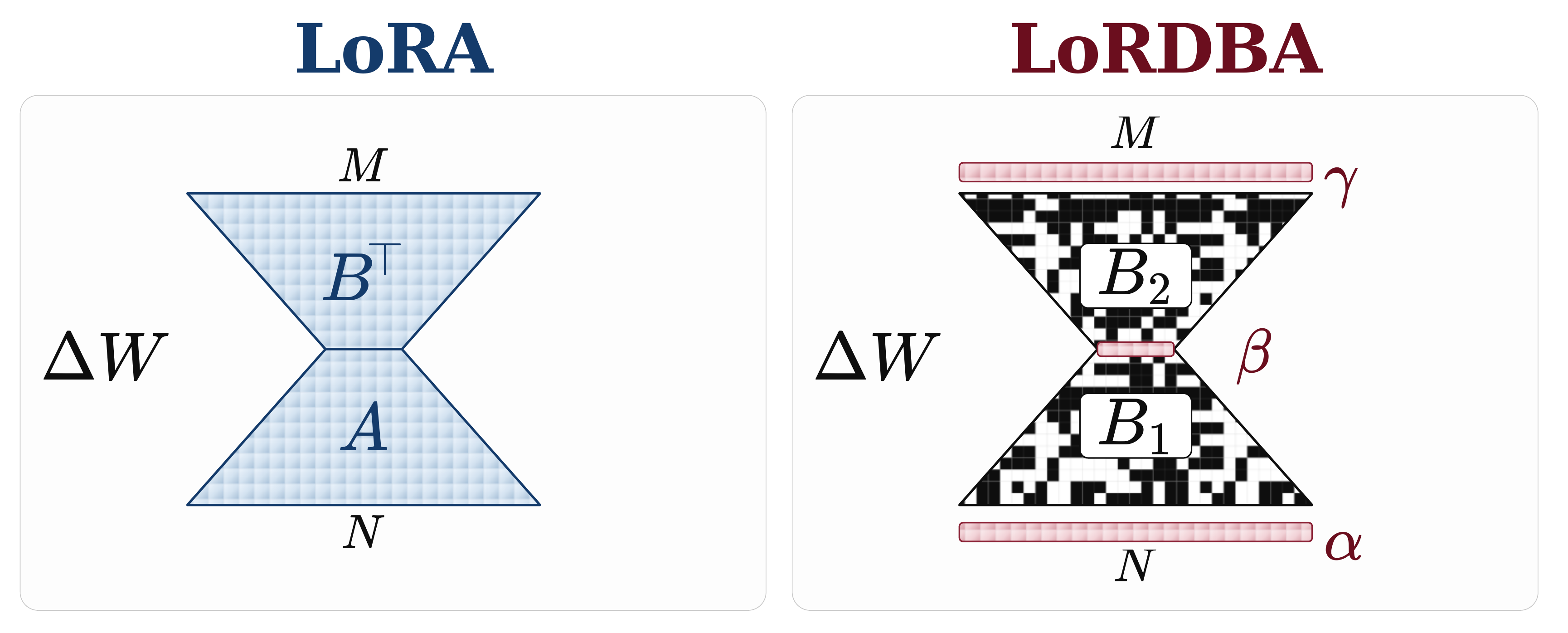}
    \caption{Adapter-mode LoRA versus LoRDBA. The base model is stored in a low-bit codebook and executed with a fused dequantization--matrix multiplication kernel, while the adapter is served in an unmerged form. (a) Standard LoRA introduces two dense fp16 factors and requires a separate fp16 matrix multiplication branch. (b) LoRDBA replaces both factors with sign matrices coupled by three channel-wise scale vectors, reducing the adapter forward pass to two sign-accumulation matrix multiplications.}
    \label{fig:concept}
\end{figure}

We introduce LoRDBA, a low-rank double-binary adapter designed for unmerged adapter-mode deployment. LoRDBA replaces both LoRA factors with binary sign carriers and augments them with a compact scale envelope along the input, rank, and output axes. Binary signs encode directional information while lightweight channel-wise scales capture magnitudes; the quantized base model remains frozen and unchanged.
LoRDBA is trained end-to-end using a smooth-sign straight-through estimator and exported as a near-binary adapter. At inference, it replaces the dense adapter branch with binary-weight accumulation interleaved with scaling. We analyze when this binary reconstruction is accurate, showing that the key quantity is the residual-to-magnitude ratio of the floating-point LoRA factors.
Unlike BitDelta~\citep{liu2024bitdelta}, which binarizes the full $\Theta(NM)$ delta discarding low-rank structure, and LoRAQuant~\citep{mirzaei2025loraquant}, which retains mixed-precision SVD components, LoRDBA commits \emph{both} rank-$R$ factors to $\pm 1$ and trains end-to-end, preserving the low-rank adapter format.

We evaluate LoRDBA under the adapter-mode protocol---adapter never merged; storage, latency, and training memory measured separately. Across LLaMA-based math reasoning and summarization benchmarks, LoRDBA improves the accuracy--size trade-off over all evaluated low-bit baselines as previewed in Figure~\ref{fig:intro-preview}.

\section{Preliminaries}\label{sec:prelim}

\paragraph{Notation.} 
Uppercase Roman letters denote matrices; lowercase boldface denotes vectors. $\{\pm 1\}^{p \times q}$ is the set of sign matrices. $\|\cdot\|_F$ and $\|\cdot\|_{\text{op}}$ are the Frobenius and operator norms. $\diag(\B{v})$ constructs a diagonal matrix; $\sign(U)$ is entry-wise sign with $\sign(0){=}{+}1$. The sub-Gaussian Orlicz norm is $\|X\|_{\psi_2} \coloneqq \inf \{ t > 0 : \mab{E}[\exp(X^2/t^2)] \le 2 \}$~\citep{vershynin2018high}.

\paragraph{LoRA.} 
Given a frozen effective weight matrix $W_{0} \in \mab{R}^{N\times M}$, LoRA reparameterizes the fine-tuning increment as
\begin{equation}
    \Delta W  =  AB^{\top} \in \mab{R}^{N\times M},~~
    A \in \mab{R}^{N\times r},~B \in \mab{R}^{M\times r}, 
\label{eq:lora}
\end{equation}
where $r \ll \min(N,M)$ is the adapter rank. Following the column-vector convention, the adapted forward pass for an input $\B{x} \in \mab{R}^{N}$ is
$\B{y} = W_{0}^{\top}\B{x} + B(A^{\top}\B{x}) \in \mab{R}^{M}$. Thus, each adapted projection incurs two adapter GEMMs with an aggregate arithmetic complexity of $\mathcal{O}(r(N+M))$ per token and stores $16r(N+M)$ bits when both factors are kept in fp16. The base weights remain fixed throughout adaptation.

\paragraph{Adapter-mode deployment.}
We assess LoRDBA in a QLoRA-style on-device regime \citep{dettmers2023qlora, sheng2024slora, chen2024punica}. Let $W_{\mathrm{fp}}$ denote the original full-precision pretrained weight, and let $W_{Q}$ be its stored low-bit codebook representation. The effective deployed base weight is
\begin{equation*}
    \widetilde W_{0} \coloneqq \mathrm{dequant}(W_{Q}),
\end{equation*}
which generally differs from $W_{\mathrm{fp}}$ due to quantization error. In this setting, the base model is never materialized in fp16, and the trained adapter is served in an unmerged state. The adapted forward pass is
\begin{equation}
    \B{y}  =  \widetilde W_{0}^{\top}\B{x}  +  \Delta W^{\top}\B{x}
    ~=~ \mathrm{dequant}(W_{Q})^{\top}\B{x}  +  \Delta W^{\top}\B{x}.
    \label{eq:adapter-mode-fwd}
\end{equation}
This factorizes the operation into a base GEMM---executed by a fused dequantize-matmul kernel when the base is quantized---and an independent adapter side branch. The device-side workload is characterized by the following four conditions:

\begin{enumerate}[label=(A\arabic*),leftmargin=*]
    \item \textbf{Frozen base:} the base model is held fixed and is not updated. In typical on-device deployments it is stored in a low-bit format such as NF4 or GPTQ, but the conditions themselves do not mandate a specific base precision.
    \item \textbf{Adapter-only training:} optimizer states for the base weights are never materialized.
    \item \textbf{unmerged serving:} the adapter is evaluated as a side branch rather than being merged into a task-specific base weight.
    \item \textbf{Hot-swappability:} multiple task-specific adapters can share the same on-device base \citep{sheng2024slora, chen2024punica}.
\end{enumerate}

The unmerged regime decouples base and adapter kernels: the base GEMM runs at the codebook's native precision while the adapter GEMM runs at its own bit width. Reducing the adapter bit width is thus the primary mechanism for improving the bandwidth-compute envelope without violating the no-merging constraint.

\paragraph{Bits per weight (BPW).}
Following standard practice~\citep{mirzaei2025loraquant}, we quantify adapter compression in \emph{bits per weight} (BPW) relative to the reference LoRA parameter count $r_{0}(N + M)$. An fp16 LoRA adapter of rank $r$ thus has BPW $= 16r/r_{0}$, and the full-size reference at rank $r_{0}$ has BPW $= 16$. The LoRDBA-specific BPW metrics are defined in Section~\ref{subsec:lordba-def} once the adapter parameterization is introduced.

\section{Method}\label{sec:method}

\subsection{Low Rank Double Binary Adapter}\label{subsec:lordba-def}

LoRDBA parameterizes the LoRA delta with two binary factors and a small collection of channel-wise fp16 scale vectors across the input, rank, and output axes. This design follows the multi-envelope principle introduced for full-weight extreme quantization in MDBF~\citep{ichikawa2025more}, but applies it to the much smaller low-rank LoRA delta. Since a single set of scales is a fragile one-dimensional summary of the rank axis, we allow the adapter to contain multiple envelopes that share the same binary carriers, with the number of envelopes controlled by an integer $\ell \ge 1$.

\begin{definition}[Low Rank Double Binary Adapter]
    \label{def:lordba}
    Fix a reference fp16 LoRA rank $r_{0} \ge 1$ used only for normalizing storage. A \emph{LoRDBA adapter} of binary carrier rank $R\in\mathbb{N}$ and envelope rank $\ell \ge 1$ for a host projection of shape $N \times M$ is represented as the tuple $\theta = \bigl(B_{1},B_{2},(\B{\alpha}^{(i)},\B{\beta}^{(i)},\B{\gamma}^{(i)})_{i=1}^{\ell}\bigr)$ with
    \begin{equation}
        B_{1} \in \{\pm 1\}^{N\times R}, ~ 
        B_{2} \in \{\pm 1\}^{R\times M}, ~ 
        \B{\alpha}^{(i)} \in \mab{R}^{N},~
        \B{\beta}^{(i)} \in \mab{R}^{R},~
        \B{\gamma}^{(i)} \in \mab{R}^{M},
    \end{equation}
    whose induced weight update is
    \begin{equation}
        \Delta W(\theta)
         = \sum_{i=1}^{\ell}\diag(\B{\alpha}^{(i)}) B_{1} \diag(\B{\beta}^{(i)}) B_{2} \diag(\B{\gamma}^{(i)})
         \in \mab{R}^{N\times M}.
        \label{eq:lordba}
    \end{equation}
    The adapted forward pass is $\B{y} = W_{0}^{\top}\B{x} + \Delta W(\theta)^{\top}\B{x}$
    and the exact storage cost is
    \begin{equation}
                \mathrm{bits}(\theta) = R(N + M) + 16 \ell (N + R + M).
        \label{eq:lordba-bits}
    \end{equation}
    We refer to $R$ as the \emph{binary carrier rank} (the inner dimension of the binary matmul), reserving $r_{0}$ for the fixed reference LoRA rank used only in BPW normalization.
\end{definition}

The exact storage cost of Eq.~\eqref{eq:lordba-bits} induces two BPW metrics:
\begin{equation}
    \mathrm{BPW}_{\mathrm{bc}}~\coloneqq~\frac{R}{r_{0}}, \qquad
    \mathrm{BPW}_{\mathrm{tot}}~\coloneqq~\frac{R(N + M) + 16\ell(N + R + M)}{r_{0}(N + M)}.
    \label{eq:bpw-defs}
\end{equation}
$\mathrm{BPW}_{\mathrm{bc}}$ counts only the binary carriers; $\mathrm{BPW}_{\mathrm{tot}}$ includes all fp16 scale vectors. At the default operating point $R{=}r_0, \ell{=}1$, we have $\mathrm{BPW}_{\mathrm{tot}} \to 1$ as $r_0, N, M \to \infty$. All tables report $\mathrm{BPW}_{\mathrm{tot}}$ and actual storage in MB.

The canonical $\ell = 1$, $R = r_{0}$ case reduces to $\Delta W = \diag(\B{\alpha})B_{1}\diag(\B{\beta})B_{2}\diag(\B{\gamma})$ (Figure~\ref{fig:concept}). For $\ell \ge 2$, Eq.~\eqref{eq:lordba} sums $\ell$ rank-$R$ terms; inference cost grows linearly in $\ell$, while binary-carrier storage is unchanged.

\subsection{LoRDBA Training}\label{subsec:training}

Under (A1)--(A4), LoRDBA is trained end-to-end on the downstream task with the base $\widetilde W_{0}=\mathrm{dequant}(W_Q)$ frozen. The procedure has three components: a forward pass with binarized carriers, a backward pass with a smooth-sign STE, and an SVD-based initialization. We call this pipeline \textbf{QAT Full}; all main results use it.

\paragraph{Forward and backward pass.}
We maintain real-valued latent carriers $H_{1}\in\mab{R}^{N\times R}$, $H_{2}\in\mab{R}^{R\times M}$ and binarize them in the forward pass: $B_{1}=\sign(H_{1})$, $B_{2}=\sign(H_{2})$. The weight update $\Delta W(\theta)$ is computed via Eq.~\eqref{eq:lordba}. 
On the backward pass, the non-differentiable $\sign(\cdot)$ is replaced by a smooth-sign straight-through estimator (STE)~\citep{leng2018extremely,liu2020reactnet,boza2025dsf, ichikawa2025optimization, ichikawa2024controlling} with a temperature schedule (Appendix~\ref{app:qat-ste}). Recent high-dimensional analyses of quantized-model training further indicate that STE dynamics can be strongly shaped by quantization hyperparameters such as bit width and quantization range~\citep{ichikawa2025high}.
Channel-wise scales $(\B{\alpha},\B{\beta},\B{\gamma})$ are trained in fp32 and cast to fp16 at export. The rest of the pipeline (AdamW optimizer, data pipeline, learning rate schedule) is identical to a standard LoRA training loop.

\paragraph{Initialization.}
The latent carriers are initialized from the rank-$R$ thin SVD of a short, $10\%$-budget fp16 LoRA warm-up. Given $\Delta W_{\mathrm{warm}}\approx U_R S_R V_R^{\top}$:
\begin{equation}
    H_{1}^{(0)} = U_{R},~~
    H_{2}^{(0)} = V_{R}^{\top},~~
    B_{1}^{(0)} = \sign(H_1^{(0)}),~~
    B_{2}^{(0)} = \sign(H_2^{(0)}),~~
    \B{\beta}^{(0)} = \diag(S_{R}),
    \label{eq:svd-init}
\end{equation}
with $\B{\alpha}^{(0)},\B{\gamma}^{(0)}$ recovered by one closed-form per-axis least-squares sweep, cf.\ Appendix~\ref{app:scale-updates}. For $\ell \ge 2$, additional envelopes are zero-initialized as $\B{\beta}^{(i)}{=}\B{0}$, so the initial function equals the $\ell{=}1$ warm-start. This initialization adds negligible overhead relative to the full QAT budget.

\paragraph{Parameter and storage cost.}
The trainable parameter count is $R(N{+}M)+\ell(N{+}R{+}M)$---the same leading term as a rank-$R$ LoRA adapter plus the envelope overhead. At export, only $\sign(H_1)$, $\sign(H_2)$, and the fp16 scales are stored, yielding the storage of Eq.~\eqref{eq:lordba-bits}.

\paragraph{Optional training variants.}
LoRDBA admits alternative training modes that share the same exported adapter format and inference kernel. These are evaluated as ablations in Appendix~\ref{app:training-modes}:
\begin{itemize}[leftmargin=*,itemsep=2pt,topsep=2pt]
  \item \textbf{QAT Freeze}: freeze the binary carriers $B_1,B_2$ from the PTQ initialization and train only the scale vectors $(\B{\alpha},\B{\beta},\B{\gamma})$, reducing the trainable count to $\ell(N{+}R{+}M)$ per projection. Wall-clock time is comparable to QAT Full due to the full forward--backward pass through the frozen base.
  \item \textbf{PTQ-LoRDBA}: a training-data-free scaled-consensus ADMM that fits a LoRDBA adapter to an existing fp16 LoRA $\Delta W^{\star}$ in Frobenius norm (Appendix~\ref{app:ptq-extension}). It also serves as the initialization for QAT Full.
\end{itemize}

\subsection{Adapter-Mode Inference Kernel}\label{subsec:kernel}

For an input batch $X\in\mab{R}^{T\times N}$ and a canonical $\ell{=}1$ LoRDBA adapter, the evaluation order is
\begin{equation}
    X \Delta W(\theta) = \Bigl(\Bigl(\Bigl(\bigl(X D_{\B{\alpha}}\bigr) B_{1}\Bigr) D_{\B{\beta}}\Bigr) B_{2}\Bigr) D_{\B{\gamma}},
    \label{eq:lordba-fwd}
\end{equation}
where $D_{\B{v}}\coloneqq\diag(\B{v})$. The diagonal products are element-wise fp16 rescalings; the inner matmuls $(\cdot)B_{1}$ and $(\cdot)B_{2}$ are binary-weight GEMMs. Each column reduces to sign-accumulation:
\begin{equation}
    (X B)_{tj}
      =  \sum_{i:B_{ij}=+1} X_{ti} - \sum_{i:B_{ij}=-1} X_{ti},
    \label{eq:sign-accum}
\end{equation}
which eliminates multiplications in the inner product. We implement the binary GEMMs via \texttt{gemlite}~\citep{badri2023gemlite}; the base codebook GEMM is unchanged. The bit-packed $\{B_{1},B_{2}\}$ is stored once per adapter; for $\ell\ge 2$ the same carriers are shared across envelopes. Pseudocode and microbenchmarks are in Appendices~\ref{app:kernel} and~\ref{sec:exp-speed}.

\section{Theoretical Analysis}\label{sec:theory}

We explain why replacing both fp16 LoRA factors by binary sign carriers can preserve the LoRA update when factor magnitudes are not too dispersed. The formal model and proof are in Appendix~\ref{app:expressivity}; here we state the guarantee in the compact form used to interpret the method.

\noindent\textbf{Intuition.}~~If the entries of each LoRA factor are well-separated from zero (i.e., dominated by their signs), then replacing the factors by their element-wise signs introduces a relative Frobenius error proportional to the residual-to-magnitude ratio $\zeta/\mu$.

\begin{theorem}[Informal LoRDBA expressivity]
    \label{thm:expressivity}
    Let $A\in\mathbb{R}^{N\times r}$ and $B\in\mathbb{R}^{M\times r}$ be fp16 LoRA factors. Suppose that, after fixing the usual positive column-rescaling gauge of LoRA, their entries can be written as
    \begin{equation*}
        A_{ik}=\mu_A\sigma^A_{ik}+\xi^A_{ik},
        \qquad
        B_{jk}=\mu_B\sigma^B_{jk}+\xi^B_{jk},
    \end{equation*}
    where $\mu_A,\mu_B>0$, the sign arrays are in $\{\pm1\}$, and the residuals are independent, mean-zero, sub-Gaussian with scale at most $\zeta$, and independent of the signs. For the relative bound, assume the signs are jointly i.i.d. Rademacher, $NM>8$, $\delta\in(0,1-8/(NM))$, $\zeta\le\max(\mu_A,\mu_B)$, and $\log(6NM/\delta)\le c_1r$ for a universal constant $c_1>0$.

    Let $\theta^\star$ be the canonical single-envelope LoRDBA adapter with binary carriers $\sigma^A$ and $(\sigma^B)^\top$, unit input/output scales, and rank-axis scale $\mu_A\mu_B\mathbf{1}_r$. Then, for a universal constant $C'>0$, with probability at least $1-\delta-8/(NM)$,
    \begin{equation}
        \frac{\|AB^{\top}-\Delta W(\theta^{\star})\|_{F}}
        {\|\Delta W(\theta^{\star})\|_{F}}
        \le
        C'\frac{\zeta}{\min(\mu_A,\mu_B)}\sqrt{\log(2NM/\delta)}.
        \label{eq:relative-bound}
    \end{equation}
    The full statement, including the absolute Frobenius bound, the observed-sign version using $\sign(A)$ and $\sign(B)$, and the monotonic extension to envelope rank $\ell\ge1$, is given in Appendix~\ref{app:expressivity}.
\end{theorem}

The bound is controlled by $\zeta/\mu$: the residual-to-magnitude ratio after separating signs from factor entries. In the low-noise regime, binary carriers preserve the Frobenius signal up to logarithmic factors, predicting a performance plateau near the binary-carrier limit. Appendix~\ref{app:signnoise-diagnostic} reports the plug-in diagnostic.

\paragraph{Role of the bound in the QAT pipeline.}
Theorem~\ref{thm:expressivity} directly characterises PTQ reconstruction quality: given fp16 LoRA factors whose entries are well-separated from zero, the binary sign carriers preserve most of the Frobenius signal.
In the QAT Full pipeline (Section~\ref{subsec:training}), this PTQ solution serves as the initialization; the subsequent smooth-sign training refines both the binary carriers and the scales on the downstream loss, which explains why QAT Full consistently outperforms PTQ-LoRDBA (Appendix~\ref{app:training-modes}).
The theorem thus plays two complementary roles: (i)~it identifies the residual-to-magnitude ratio $\zeta/\min(\mu_A,\mu_B)$ as the diagnostic governing initialization quality, and (ii)~it predicts a diminishing-returns plateau near the binary-carrier limit, beyond which spending more bits on dense factors yields little improvement---a prediction confirmed by the BPW sweep of Section~\ref{sec:exp-ablation}.

\section{Experiments}\label{sec:exp}

Under the adapter-mode conditions (A1)--(A4), we evaluate LoRDBA along the three axes that jointly determine deployment viability:
\begin{itemize}[itemsep=1pt,topsep=3pt]
  \item[\textbf{RQ1}] \textbf{Accuracy} (Section~\ref{sec:exp-accuracy}): Does LoRDBA match or exceed fp16 LoRA at matched adapter size?
  \item[\textbf{RQ2}] \textbf{Inference speed} (Section~\ref{sec:exp-speed}): Does the binary-weight kernel reduce unmerged adapter latency?
  \item[\textbf{RQ3}] \textbf{Training cost} (Section~\ref{sec:exp-memory}): What is the peak memory and wall-clock overhead relative to the corresponding fp16 LoRA loop?
\end{itemize}

\paragraph{Evaluation protocol.}
We follow the adapter-mode plug-and-play protocol~\citep{liu2024bitdelta,mirzaei2025loraquant}. A reference fp16 LoRA of rank $r_0$ is first fine-tuned on the downstream task atop the frozen base. The candidate method then emits an adapter at the target BPW under (A1)--(A4), and the adapter is evaluated \emph{without merging}. Accuracy is compared at matched adapter size, latency on the unmerged forward pass, and peak memory on the same training device. LoRDBA trains end-to-end rather than compressing a pre-trained LoRA; the optional ADMM refinement is in Appendix~\ref{app:ptq-extension}.

\paragraph{Setup}

All experiments instantiate (A1)--(A4): the base is frozen in bf16, only the adapter is trainable, and the adapter is served unmerged. 
We use bf16 rather than a quantized base to isolate adapter-compression effects from base-quantization noise; LoRDBA is orthogonal to the base format and composes with base-side PTQ pipelines, including layer-wise error-propagation and submodule-level quantization methods~\citep{arai2025quantization,ichikawa2025lpcd}, since it only modifies the adapter branch.

\textbf{Models.}~~\textbf{LLaMA-2-7B} is the primary base; \textbf{LLaMA-3.2-3B} serves as an ablation base. The reference LoRA targets all seven Transformer projections $\{q,k,v,o,\mathrm{gate},\mathrm{up},\mathrm{down}\}_{\mathrm{proj}}$ with $\alpha{=}16$ and rank $r_0{=}16$; in ablations $r_0 \in \{1,2,4,64\}$.

\textbf{Tasks and metrics.}~~We fine-tune task-specific adapters on \textsc{MetaMathQA}, \textsc{Magicoder-Evol-Instruct-110K}, and \textsc{EdinburghNLP/xsum}, and evaluate on \textsc{GSM8K} with $8$-shot, \textsc{Minerva Math} with $4$-shot, and \textsc{XSum} via ROUGE-L F1, all through \texttt{lm-evaluation-harness}~\citep{eval-harness}.

\textbf{Baselines.}~~\textbf{LoRA fp16} at several ranks, \textbf{LoRAQuant}~\citep{mirzaei2025loraquant}, and \textbf{BitDelta}~\citep{liu2024bitdelta} (scalar, column, column$+$distill, and row variants). All post-hoc compressors operate on the same task-specific fp16 adapter.

\textbf{Hardware.}~~Training and evaluation use NVIDIA H100 (80\,GB) and B200 GPUs. The latency benchmark (Table~\ref{tab:inference-speed}) runs single-stream on a single H100 NVL. Full hyper-parameter and hardware details are in Appendix~\ref{app:setup}.

\textbf{Compute budget.}~~Candidates are matched at equal adapter storage, not equal training compute. QAT Full requires $1.5$--$1.8\times$ total GPU-hours vs.\ the fp16 LoRA baseline (one-time cost; details in Appendix~\ref{app:setup}). QAT Freeze offers near-equivalent accuracy at comparable wall-clock cost (Appendix~\ref{app:training-modes}).

\subsection{Adapter Accuracy (RQ1)}\label{sec:exp-accuracy}

\begin{table*}[t]
\centering
\caption{LLaMA-2-7B adapter-mode accuracy at matched size. \emph{Rank} denotes the binary carrier rank $R$, cf.\ Definition~\ref{def:lordba}; for uncompressed LoRA rows, it coincides with the standard LoRA rank. BPW is $\mathrm{BPW}_{\mathrm{tot}}$ of Eq.~\ref{eq:bpw-defs}; Avg averages GSM8K, Minerva Math, and XSum. The full-size LoRA FP16 reference at $r_0{=}16$, 76.3\,MB, BPW${}=16$ achieves Avg $28.81$. Optional training variants are compared in Appendix~\ref{app:training-modes}.}
\label{tab:7b-size-ranking}
\resizebox{\textwidth}{!}{%
\begin{tabular}{llrrrrrr}
\toprule
\textbf{Method} & \textbf{Rank} & \textbf{BPW} & \textbf{Size (MB)} & \textbf{GSM8K} & \textbf{Minerva} & \textbf{XSum} & \textbf{Avg} \\
\midrule
\multicolumn{8}{l}{\textit{Full-size reference (upper bound)}} \\
LoRA FP16                     & $r{=}16$ & 16.0 &  76.3 & 54.38 & 13.17 & 18.87 & 28.81 \\
\midrule
\multicolumn{8}{l}{\textit{$\sim$7 MB tier}} \\
LoRA FP16 (same-size ref.)    & $r{=}2$  & 2.0  &   9.5 & 43.52 &  8.26 & 17.85 & 23.21 \\
\qcol \textbf{QAT Full (Ours)} & $r{=}16$ & 1.0  &   7.2 & \textbf{44.43} & \textbf{9.68} & \textbf{18.32} & \textbf{24.14} \\
LoRAQuant $\rho{=}0.9$           & $r{=}16$ & 1.68 &   8.0 & 42.30 &  8.94 & 18.27 & 23.17 \\
LoRAQuant $\rho{=}0.5$           & $r{=}16$ & 1.20 &   5.7 & 36.09 &  7.18 & 17.67 & 20.31 \\
\midrule
\multicolumn{8}{l}{\textit{$\sim$9.5 MB tier}} \\
LoRA FP16 (same-size ref.)    & $r{=}2$  & 2.0  &   9.5 & 43.52 &  8.26 & 17.85 & 23.21 \\
\qcol \textbf{QAT Full (Ours)} & $r{=}16$ & 2.0  &   9.5 & \textbf{48.60} & \textbf{9.68} & \textbf{18.49} & \textbf{25.59} \\
\midrule
\multicolumn{8}{l}{\textit{$\sim$19 MB tier}} \\
LoRA FP16 (same-size ref.)    & $r{=}4$  & 4.0  &  19.1 & 48.07 & 10.28 & 18.23 & 25.53 \\
\qcol \textbf{QAT Full (Ours)} & $r{=}64$ & 1.0  &  19.1 & \textbf{53.60} & 12.04 & 18.60 & \textbf{28.08} \\
\qcol QAT Full (Ours)          & $r{=}16$ & 4.0  &  19.1 & 53.45 & \textbf{12.28} & \textbf{18.49} & 28.07 \\
\bottomrule
\end{tabular}%
}
\end{table*}

Table~\ref{tab:7b-size-ranking} compares QAT Full LoRDBA against baselines at three matched adapter-size tiers on LLaMA-2-7B.

\paragraph{Result.}
LoRDBA outperforms all baselines at every size tier, where Avg $\coloneqq$ (GSM8K + Minerva + XSum) / 3:
\begin{itemize}[leftmargin=*,itemsep=2pt,topsep=2pt]
  \item \textbf{${\sim}7$\,MB tier.}~~QAT Full @1bpw, achieving Avg $24.14$ at $7.2$\,MB, Pareto-dominates both LoRAQuant $\rho{=}0.9$ with Avg $23.17$ at $8.0$\,MB and LoRA FP16 $r{=}2$ with Avg $23.21$ at $9.5$\,MB.
  \item \textbf{${\sim}9.5$\,MB tier.}~~QAT Full @2bpw with Avg $25.59$ exceeds FP16 $r{=}2$ by $+2.4$\,pt at the same storage.
  \item \textbf{${\sim}19$\,MB tier.}~~QAT Full $r{=}64$ @1bpw with Avg $28.08$ attains $97.5\%$ of the full-size FP16 reference Avg of $28.81$ at $76.3$\,MB, i.e.\ $4\times$ compression.\footnote{The recovery ratio is the absolute Avg ratio $28.08/28.81$; it does not measure improvement over a no-adapter baseline.} Notably, $r{=}64$ @1bpw and $r{=}16$ @4bpw achieve near-identical Avg---$28.08$ vs.\ $28.07$---at the same budget, consistent with the plateau predicted by Theorem~\ref{thm:expressivity}.
\end{itemize}
\noindent The improvement is concentrated on math reasoning, namely GSM8K and Minerva; among the methods evaluated on XSum in Table~\ref{tab:7b-size-ranking}, scores are relatively flat at $17$--$18$ ROUGE-L, so the Avg differences are primarily driven by the math benchmarks.
The matched comparison on LLaMA-3.2-3B in Appendix~\ref{app:extended-tables} preserves the same ordering.

\subsection{unmerged Adapter Inference Speed (RQ2)}\label{sec:exp-speed}

We ask whether the binary adapter branch introduces latency overhead relative to fp16 LoRA in the unmerged regime.

\begin{table}[tbh]
    \centering
    \caption{Prefill latency on a single NVIDIA H100 NVL with \textsc{LLaMA-2-7B}, prompt length $= 56$ tokens. Unmerged rows execute the adapter as a side branch over a frozen $4$-bit base. The LoRDBA kernel at $\ell{=}1$ is a fused WGMMA CUDA kernel for sign-accumulation matmuls.}
    \label{tab:inference-speed}
    \setlength{\tabcolsep}{6pt}
    \renewcommand{\arraystretch}{0.95}
    \small
    \begin{tabular}{@{}l c c@{}}
    \toprule
    Config & Size (MB) & Prefill (ms) \\
    \midrule
    Base model only                                              & --      & $19.3$ \\
    \midrule
    LoRA FP16 $r{=}1$, unmerged                                  &  $4.8$  & $32.7$ \\
    LoRA FP16 $r{=}2$, unmerged                                  &  $9.5$  & $33.4$ \\
    LoRA FP16 $r{=}16$, unmerged                                 & $76.2$  & $33.5$ \\
    \midrule
    \qcol \textbf{LoRDBA $\ell{=}1$, unmerged}    &  $\mathbf{7.2}$ & $\mathbf{36.1}$ \\
    \bottomrule
    \end{tabular}
\end{table}

Table~\ref{tab:inference-speed} shows that unmerged LoRDBA adds only $2.6$\,ms, or $8\%$, to the prefill latency of unmerged fp16 LoRA at $r{=}16$---$36.1$\,ms vs.\ $33.5$\,ms---while compressing the adapter by $10.6\times$, from $76.2$\,MB to $7.2$\,MB. Even compared with the closest-size fp16 baseline, LoRA $r{=}1$ at $4.8$\,MB, LoRDBA adds only $3.4$\,ms or $10\%$ at $1.5\times$ the adapter size. The binary kernel does not achieve a net speed-up at prefill batch size $T{=}56$ because the base-model GEMM dominates total latency; the adapter bandwidth saving becomes more pronounced when multiple adapters are served concurrently or at larger adapter ranks.

\paragraph{Analysis.}
The theoretical adapter-bandwidth ratio at $\ell{=}1, R{=}r_0$ is
\begin{equation}
    \frac{16 r_{0}(N{+}M)}{r_{0}(N{+}M) + 16(N{+}r_{0}{+}M)}
    ~\approx~ 8.0\times~(r_0{=}16),
    \quad 12.8\times~(r_0{=}64),
    \quad \to 16\times~(r_0 \to \infty),
\end{equation}

which upper-bounds the wall-clock speedup when the adapter matmul is bandwidth-bound. The end-to-end measurement confirms that, even outside that regime, the fused LoRDBA kernel preserves unmerged fp16 LoRA latency at $10\times$ smaller adapter size.

\subsection{On-Device Training Memory (RQ3)}\label{sec:exp-memory}

We compare LoRDBA QAT Full to standard LoRA fp16 training in peak GPU memory and wall-clock time. Table~\ref{tab:training-memory} reports measurements on a single NVIDIA B200.

\paragraph{Result.}
The overhead is moderate and practical:
\begin{itemize}[leftmargin=*,itemsep=2pt,topsep=2pt]
  \item \textbf{Peak memory:}~~$58.0$\,GB (@4BPW) and $55.6$\,GB (@1BPW) vs.\ $36.3$\,GB for LoRA fp16, i.e.\ approximately $1.5$--$1.6\times$. The overhead is due to the smooth-sign STE keeping both latent and binarized copies of the binary carriers plus their Adam states; @4BPW is slightly larger due to higher carrier rank.
  \item \textbf{Wall-clock time:}~~$718$--$744$\,min total, approximately $1.9\times$ the $385$\,min LoRA baseline. Per-step time of $804$--$899$\,ms is $1.7$--$1.9\times$ that of LoRA fp16 at $468$\,ms.
\end{itemize}
\noindent
When the training budget is limited, QAT Freeze offers an effective trade-off between accuracy and efficiency. By fixing the binary signs obtained from PTQ and updating only the channel-wise scales, it matches QAT Full within $0.02$ average accuracy points at 2\,BPW, while reducing optimizer-state memory through a substantially smaller set of trainable parameters. Detailed wall-clock and memory comparisons across all training modes are reported in Appendix~\ref{app:training-modes}.

\begin{table}[t]
    \centering
    \caption{On-device training peak memory, step time, and total wall-clock time on
    \textsc{LLaMA-2-7B}+\textsc{MetaMathQA} (single NVIDIA B200).
    LoRA fp16 trains for $2$ epochs; QAT Full adds $1$-epoch STE refinement
    after PTQ initialization. All rows freeze the bf16 base and
    update the adapter only; LoRDBA uses the smooth-sign STE.
    \emph{Time} = total wall-clock: LoRA pre-training $+$ PTQ init $+$ QAT fine-tuning.}
    \label{tab:training-memory}
    \setlength{\tabcolsep}{5pt}
    \renewcommand{\arraystretch}{0.95}
    \small
    \resizebox{\columnwidth}{!}{%
    \begin{tabular}{@{}l c c c c@{}}
    \toprule
    Method & Size (MB) & Peak (GB) & ms/step & Time (min) \\
    \midrule
    LoRA fp16, $r{=}16$                                                  & $76.3$           & $36.3$ & $468$        & $385$ \\
    \qcol \textbf{LoRDBA QAT Full $@4$BPW}, $r{=}16,\ell{=}1$        & $19.1$           & $58.0$ & $899$        & $744$ \\
    \qcol LoRDBA QAT Full $@1$BPW, $r{=}16,\ell{=}1$                  & $7.2$            & $55.6$ & $804$        & $718$ \\
    \bottomrule
    \end{tabular}%
    }
\end{table}

\subsection{Ablations}\label{sec:exp-ablation}

\paragraph{BPW sweep.}
Figure~\ref{fig:ablation} sweeps the adapter bit budget on the \textsc{LLaMA-3.2-3B} ablation base. Since the carrier term $R/r_{0}$ dominates $\mathrm{BPW}_{\mathrm{tot}}$ (Eq.~\ref{eq:bpw-defs}) at practical dimensions, the binary-carrier limit $\mathrm{BPW}_{\mathrm{bc}}{=}1$ controls the curve's knee: trimming the carrier rank costs ${\approx}1$\,pt GSM8K, while spending more bits beyond $\mathrm{BPW}_{\mathrm{tot}}{\approx}1.5$ yields no further gain.

\paragraph{Envelope rank $\ell \ge 2$.}
Table~\ref{tab:train_stats} in the Appendix reports $\ell=2$ ablations. At $\ell{=}2, R{=}32$, corresponding to $19.1$\,MB, GSM8K reaches $47.84\%$ with only $25.0$\,GB peak memory---roughly half the $\ell{=}1$ variant---at a $5.6$\,pt accuracy cost vs.\ $\ell{=}1$ @1bpw $R{=}64$ at $53.60\%$. The $\ell{=}2, R{=}4$ configuration at $7.2$\,MB and $24.5$\,GB retains $44.35\%$ GSM8K, competitive with $\ell{=}1$ @1bpw while halving training memory, demonstrating a practical memory--accuracy dial for constrained devices.

\paragraph{Training mode variants.}
Appendix~\ref{app:training-modes} (Table~\ref{tab:ablation-variants}) compares QAT Freeze, QAT Scratch, and PTQ-LoRDBA against QAT Full. QAT Freeze matches QAT Full within $0.02$ Avg at 2\,BPW. QAT Scratch trails by $4$--$9$\,pt on GSM8K due to missing task-informed initialization. PTQ-LoRDBA degrades at $\leq 2$\,BPW on math but retains summarization quality.

\section{Related Work}\label{sec:related}

\paragraph{Low-rank adaptation.}
LoRA~\citep{hu2022lora} reparameterizes fine-tuning updates as a low-rank product of dense fp16 factors. Extensions improve rank structure~\citep{liu2024dora,kopiczko2024vera}, initialization~\citep{meng2024pissa,hayou2024loraplus}, or training, but all keep full-precision factors. QLoRA~\citep{dettmers2023qlora} freezes a $4$-bit base and trains fp16 LoRA; S-LoRA~\citep{sheng2024slora} and Punica~\citep{chen2024punica} operationalize multi-tenant adapter serving. LoRDBA compresses the \emph{trained} adapter and composes with any LoRA variant.

\paragraph{Base-model quantization.}
LoftQ~\citep{li2024loftq} and QA-LoRA~\citep{xu2024qalora} quantize the frozen base while initializing adapters to compensate for quantization error; adapters remain floating-point. Base-model PTQ methods~\citep{frantar2022gptq,lin2024awq} are orthogonal to LoRDBA. Recent PTQ methods further improve the base-side quantization objective by explicitly propagating quantization errors across layers~\citep{arai2025quantization} or by extending the optimization unit from individual layers to larger submodules~\citep{ichikawa2025lpcd}. LoRDBA instead compresses the \emph{adapter} while reusing the base-side codebook kernel without modification.

\paragraph{Adapter compression.}
BitDelta~\citep{liu2024bitdelta} binarizes the full fine-tuning delta $\Delta W$ with per-column scales; it discards the low-rank structure, producing a $\Theta(NM)$-bit delta much larger than a LoRA adapter at small rank. LoRAQuant~\citep{mirzaei2025loraquant} re-decomposes $AB^{\top}$ via SVD into mixed-bit sub-LoRAs with gradient optimization, retaining floating-point components. LoRDBA differs from both by committing both factors to $\pm 1$ and training end-to-end.

\paragraph{Binary matrix factorization.}
The $\pm 1$ parameterization connects to binary networks~\citep{courbariaux2016binarized,rastegari2016xnor,liu2020reactnet} and full-LLM binarization (BitNet~\citep{wang2023bitnet}, OneBit~\citep{xu2024onebit}). 
DBF~\citep{boza2026dbf} factorizes a dense weight as two sign matrices with channel-wise scales; MDBF~\citep{ichikawa2025more} extends this idea to rank-$\ell$ multi-envelope scaling for full-weight extreme quantization. LoRDBA applies the factorization to the much smaller LoRA delta, adds an expressivity guarantee (Theorem~\ref{thm:expressivity}), and introduces task-specific QAT via the smooth-sign STE~\citep{leng2018extremely,boza2025dsf}.
Complementary to algorithmic compression methods, OneComp~\citep{ichikawa2026onecomp} packages heterogeneous post-training compression techniques into a reproducible, hardware-aware pipeline, whereas LoRDBA focuses on the adapter representation and its unmerged inference primitive.

\section{Conclusion}\label{sec:conclusion}
LoRDBA shows that both factors of a low-rank adapter can be binarized to $\pm 1$, provided that a compact set of channel-wise scales preserves their row, column, and joint-spectrum magnitudes. Under the QLoRA-style adapter-mode conditions (A1)--(A4), QAT Full LoRDBA achieves accuracy that matches or exceeds fp16 LoRA at the same adapter size, while introducing at most $8\%$ prefill-latency overhead, and compressing the adapter by more than $10\times$.
The additional training cost is moderate: LoRDBA requires approximately $1.6\times$ higher peak memory and $1.8\times$ longer step time due to the smooth-sign STE. Across LLaMA-2-7B math reasoning and summarization benchmarks, as well as LLaMA-3.2-3B ablations, QAT Full LoRDBA Pareto-dominates all evaluated adapter-compression methods at every tested size tier.

\paragraph{Limitations.}
Our expressivity guarantee in Theorem~\ref{thm:expressivity} relies on the sub-Gaussian sign-noise decomposition in Assumption~\ref{asm:signnoise}. As a result, the guarantee may weaken when the factor magnitudes are heavy-tailed or bimodal. Our experiments are also limited to LLaMA-2-7B and LLaMA-3.2-3B on math reasoning and summarization tasks, leaving validation on larger models, alternative architectures, and additional domains to future work.
Finally, LoRDBA targets the unmerged adapter-mode setting. Adapters that must be merged into the base model before inference are outside the scope of the current framework.

\paragraph{Broader impact.}
By compressing LoRA adapters to ${\leq}4$ bits per weight, LoRDBA reduces the communication and storage costs of distributing task-specific models. This makes on-device personalization more practical, as adapters can be deployed locally without requiring users to upload private data to the cloud, providing a direct privacy benefit.
At the same time, LoRDBA is both model- and task-agnostic, and could therefore also be used to compress adapters fine-tuned for harmful applications. We recommend pairing LoRDBA with the same safety guardrails used for the underlying base model.

\begin{ack}
This work was supported by the Council for Science, Technology and Innovation (CSTI), Cross-ministerial Strategic Innovation Promotion Program (SIP), “Promotion of Application of Advanced Quantum Technology Infrastructure to Social Issues” (Funding agency: QST),  JST BOOST (Grant No. JPMJBY24D0), and was carried out using the TSUBAME4.0 supercomputer at Institute of Science Tokyo.
\end{ack}

\bibliographystyle{plainnat}
\bibliography{ref}

\begin{thebibliography}{38}
\providecommand{\natexlab}[1]{#1}
\providecommand{\url}[1]{\texttt{#1}}
\expandafter\ifx\csname urlstyle\endcsname\relax
  \providecommand{\doi}[1]{doi: #1}\else
  \providecommand{\doi}{doi: \begingroup \urlstyle{rm}\Url}\fi

\bibitem[Arai and Ichikawa(2025)]{arai2025quantization}
Yamato Arai and Yuma Ichikawa.
\newblock Quantization error propagation: Revisiting layer-wise post-training quantization.
\newblock In \emph{The Thirty-ninth Annual Conference on Neural Information Processing Systems}, 2025.

\bibitem[Attouch et~al.(2010)Attouch, Bolte, Redont, and Soubeyran]{attouch2010proximal}
H.~Attouch, J.~Bolte, P.~Redont, and A.~Soubeyran.
\newblock Proximal alternating minimization and projection methods for nonconvex problems: An approach based on the {K}urdyka--{\l}ojasiewicz inequality.
\newblock \emph{Mathematics of Operations Research}, 35\penalty0 (2):\penalty0 438--457, 2010.

\bibitem[Attouch et~al.(2013)Attouch, Bolte, and Svaiter]{attouch2013convergence}
H.~Attouch, J.~Bolte, and B.~F. Svaiter.
\newblock Convergence of descent methods for semi-algebraic and tame problems: Proximal algorithms, forward--backward splitting, and regularized {G}auss--{S}eidel methods.
\newblock \emph{Mathematical Programming}, 137\penalty0 (1--2):\penalty0 91--129, 2013.

\bibitem[Badri and Shaji(2023)]{badri2023gemlite}
Hicham Badri and Appu Shaji.
\newblock {gemlite}: Cuda kernels for low-bit matrix multiplication, 2023.
\newblock \url{https://github.com/mobiusml/gemlite}.

\bibitem[Bolte et~al.(2007)Bolte, Daniilidis, and Lewis]{boltedaniliidislewis2007clarke}
J.~Bolte, A.~Daniilidis, and A.~Lewis.
\newblock The {{\L}}ojasiewicz inequality for nonsmooth subanalytic functions with applications to subgradient dynamical systems.
\newblock \emph{SIAM Journal on Optimization}, 17\penalty0 (4):\penalty0 1205--1223, 2007.

\bibitem[Boyd et~al.(2011)Boyd, Parikh, Chu, Peleato, and Eckstein]{boyd2011admm}
Stephen Boyd, Neal Parikh, Eric Chu, Borja Peleato, and Jonathan Eckstein.
\newblock Distributed optimization and statistical learning via the alternating direction method of multipliers.
\newblock \emph{Foundations and Trends in Machine Learning}, 3\penalty0 (1):\penalty0 1--122, 2011.

\bibitem[Bo{\v{z}}a and Macko(2025{\natexlab{a}})]{boza2025dsf}
Vladim{\'i}r Bo{\v{z}}a and Vladim{\'i}r Macko.
\newblock Two sparse matrices are better than one: Sparsifying neural networks with double sparse factorization.
\newblock In \emph{International Conference on Learning Representations (ICLR)}, 2025{\natexlab{a}}.

\bibitem[Bo{\v{z}}a and Macko(2025{\natexlab{b}})]{boza2026dbf}
Vladim{\'i}r Bo{\v{z}}a and Vladim{\'i}r Macko.
\newblock Addition is almost all you need: Compressing large language models with double binary factorization.
\newblock In \emph{International Conference on Machine Learning (ICML)}, 2025{\natexlab{b}}.

\bibitem[Chen et~al.(2024)Chen, Ye, Wu, Zhuo, Ceze, and Krishnamurthy]{chen2024punica}
Lequn Chen, Zihao Ye, Yongji Wu, Danyang Zhuo, Luis Ceze, and Arvind Krishnamurthy.
\newblock {Punica}: Multi-tenant {LoRA} serving.
\newblock In \emph{Proceedings of Machine Learning and Systems (MLSys)}, 2024.

\bibitem[Courbariaux et~al.(2016)Courbariaux, Hubara, Soudry, El-Yaniv, and Bengio]{courbariaux2016binarized}
Matthieu Courbariaux, Itay Hubara, Daniel Soudry, Ran El-Yaniv, and Yoshua Bengio.
\newblock Binarized neural networks: Training deep neural networks with weights and activations constrained to $+1$ or $-1$.
\newblock \emph{arXiv preprint arXiv:1602.02830}, 2016.

\bibitem[Dettmers et~al.(2023)Dettmers, Pagnoni, Holtzman, and Zettlemoyer]{dettmers2023qlora}
Tim Dettmers, Artidoro Pagnoni, Ari Holtzman, and Luke Zettlemoyer.
\newblock {QLoRA}: Efficient finetuning of quantized llms.
\newblock In \emph{Advances in Neural Information Processing Systems (NeurIPS)}, 2023.

\bibitem[Frantar et~al.(2022)Frantar, Ashkboos, Hoefler, and Alistarh]{frantar2022gptq}
Elias Frantar, Saleh Ashkboos, Torsten Hoefler, and Dan Alistarh.
\newblock {GPTQ}: Accurate post-training quantization for generative pre-trained transformers.
\newblock \emph{arXiv preprint arXiv:2210.17323}, 2022.

\bibitem[Gao et~al.(2023)Gao, Tow, Biderman, et~al.]{eval-harness}
Leo Gao, Jonathan Tow, Stella Biderman, et~al.
\newblock A framework for few-shot language model evaluation.
\newblock \url{https://github.com/EleutherAI/lm-evaluation-harness}, 2023.

\bibitem[Hayou et~al.(2024)Hayou, Ghosh, and Yu]{hayou2024loraplus}
Soufiane Hayou, Nikhil Ghosh, and Bin Yu.
\newblock {LoRA+}: Efficient low rank adaptation of large models.
\newblock In \emph{International Conference on Machine Learning (ICML)}, 2024.

\bibitem[Hu et~al.(2022)Hu, Shen, Wallis, Allen-Zhu, Li, Wang, Wang, and Chen]{hu2022lora}
Edward~J. Hu, Yelong Shen, Phillip Wallis, Zeyuan Allen-Zhu, Yuanzhi Li, Shean Wang, Lu~Wang, and Weizhu Chen.
\newblock Lora: Low-rank adaptation of large language models.
\newblock In \emph{International Conference on Learning Representations (ICLR)}, 2022.

\bibitem[Ichikawa(2024)]{ichikawa2024controlling}
Yuma Ichikawa.
\newblock Controlling continuous relaxation for combinatorial optimization.
\newblock In \emph{Advances in Neural Information Processing Systems}, volume~37, 2024.

\bibitem[Ichikawa and Arai(2025)]{ichikawa2025optimization}
Yuma Ichikawa and Yamato Arai.
\newblock Optimization by parallel quasi-quantum annealing with gradient-based sampling.
\newblock In \emph{The Thirteenth International Conference on Learning Representations}, 2025.

\bibitem[Ichikawa et~al.(2025{\natexlab{a}})Ichikawa, Fujimoto, and Sakai]{ichikawa2025lpcd}
Yuma Ichikawa, Yudai Fujimoto, and Akira Sakai.
\newblock {LPCD}: Unified framework from layer-wise to submodule quantization.
\newblock \emph{arXiv preprint arXiv:2512.01546}, 2025{\natexlab{a}}.

\bibitem[Ichikawa et~al.(2025{\natexlab{b}})Ichikawa, Fujisawa, Fujimoto, Sakai, and Fujisawa]{ichikawa2025more}
Yuma Ichikawa, Yoshihiko Fujisawa, Yudai Fujimoto, Akira Sakai, and Katsuki Fujisawa.
\newblock More than bits: Multi-envelope double binary factorization for extreme quantization.
\newblock \emph{arXiv preprint arXiv:2512.24545}, 2025{\natexlab{b}}.

\bibitem[Ichikawa et~al.(2025{\natexlab{c}})Ichikawa, Kashiwamura, and Sakata]{ichikawa2025high}
Yuma Ichikawa, Shuhei Kashiwamura, and Ayaka Sakata.
\newblock High-dimensional learning dynamics of quantized models with straight-through estimator.
\newblock \emph{arXiv preprint arXiv:2510.10693}, 2025{\natexlab{c}}.

\bibitem[Ichikawa et~al.(2026)Ichikawa, Kimura, Yoshida, Fujimoto, Tokura, Arai, Ishii, Kawakami, Shikada, Jacquemond, Fujisawa, Fujisawa, Honda, and Sakai]{ichikawa2026onecomp}
Yuma Ichikawa, Keiji Kimura, Akihiro Yoshida, Yudai Fujimoto, Hiroki Tokura, Yamato Arai, Yoshiyuki Ishii, Yusei Kawakami, Genki Shikada, Achille Jacquemond, Yoshihiko Fujisawa, Katsuki Fujisawa, Takumi Honda, and Akira Sakai.
\newblock {OneComp}: One-line revolution for generative {AI} model compression.
\newblock \emph{arXiv preprint arXiv:2603.28845}, 2026.

\bibitem[Kopiczko et~al.(2024)Kopiczko, Blankevoort, and Asano]{kopiczko2024vera}
Dawid~J. Kopiczko, Tijmen Blankevoort, and Yuki~M. Asano.
\newblock {VeRA}: Vector-based random matrix adaptation.
\newblock In \emph{International Conference on Learning Representations (ICLR)}, 2024.

\bibitem[Leng et~al.(2018)Leng, Dou, Li, Zhu, and Jin]{leng2018extremely}
Cong Leng, Zesheng Dou, Hao Li, Shenghuo Zhu, and Rong Jin.
\newblock Extremely low bit neural network: Squeeze the last bit out with {ADMM}.
\newblock In \emph{Proceedings of the AAAI Conference on Artificial Intelligence}, volume~32, 2018.

\bibitem[Li et~al.(2024)Li, Yu, Liang, He, Karampatziakis, Chen, and Zhao]{li2024loftq}
Yixiao Li, Yifan Yu, Chen Liang, Pengcheng He, Nikos Karampatziakis, Weizhu Chen, and Tuo Zhao.
\newblock {LoftQ}: {LoRA}-fine-tuning-aware quantization for large language models.
\newblock In \emph{International Conference on Learning Representations (ICLR)}, 2024.

\bibitem[Lin et~al.(2024)Lin, Tang, Tang, Yang, Chen, Wang, Xiao, Dang, Gan, and Han]{lin2024awq}
Ji~Lin, Jiaming Tang, Haotian Tang, Shang Yang, Wei-Ming Chen, Wei-Chen Wang, Guangxuan Xiao, Xingyu Dang, Chuang Gan, and Song Han.
\newblock {AWQ}: Activation-aware weight quantization for on-device {LLM} compression and acceleration.
\newblock \emph{Proceedings of Machine Learning and Systems}, 6, 2024.

\bibitem[Liu et~al.(2024{\natexlab{a}})Liu, Xiao, Li, Lee, Han, Dao, and Cai]{liu2024bitdelta}
James Liu, Guangxuan Xiao, Kai Li, Jason~D. Lee, Song Han, Tri Dao, and Tianle Cai.
\newblock {BitDelta}: Your fine-tune may only be worth one bit.
\newblock In \emph{Advances in Neural Information Processing Systems (NeurIPS)}, 2024{\natexlab{a}}.

\bibitem[Liu et~al.(2024{\natexlab{b}})Liu, Wang, Yin, Molchanov, Wang, Cheng, and Chen]{liu2024dora}
Shih-Yang Liu, Chien-Yi Wang, Hongxu Yin, Pavlo Molchanov, Yu-Chiang~Frank Wang, Kwang-Ting Cheng, and Min-Hung Chen.
\newblock {DoRA}: Weight-decomposed low-rank adaptation.
\newblock \emph{International Conference on Machine Learning (ICML)}, 2024{\natexlab{b}}.

\bibitem[Liu et~al.(2020)Liu, Shen, Savvides, and Cheng]{liu2020reactnet}
Zechun Liu, Zhiqiang Shen, Marios Savvides, and Kwang-Ting Cheng.
\newblock {ReActNet}: Towards precise binary neural network with generalized activation functions.
\newblock In \emph{European Conference on Computer Vision (ECCV)}, 2020.

\bibitem[Meng et~al.(2024)Meng, Wang, and Zhang]{meng2024pissa}
Fanxu Meng, Zhaohui Wang, and Muhan Zhang.
\newblock {PiSSA}: Principal singular values and singular vectors adaptation of large language models.
\newblock In \emph{Advances in Neural Information Processing Systems (NeurIPS)}, 2024.

\bibitem[Mirzaei et~al.(2025)Mirzaei, Wen, Cao, and Mou]{mirzaei2025loraquant}
Amir~Reza Mirzaei, Yuqiao Wen, Yanshuai Cao, and Lili Mou.
\newblock {LoRAQuant}: Mixed-precision quantization of {LoRA} to ultra-low bits.
\newblock \emph{arXiv preprint arXiv:2510.26690}, 2025.

\bibitem[Rastegari et~al.(2016)Rastegari, Ordonez, Redmon, and Farhadi]{rastegari2016xnor}
Mohammad Rastegari, Vicente Ordonez, Joseph Redmon, and Ali Farhadi.
\newblock {XNOR-Net}: {I}magenet classification using binary convolutional neural networks.
\newblock In \emph{European Conference on Computer Vision (ECCV)}, 2016.

\bibitem[Sheng et~al.(2024)Sheng, Cao, Li, Hooper, Lee, Yang, Chou, Zhu, Zheng, Keutzer, Gonzalez, and Stoica]{sheng2024slora}
Ying Sheng, Shiyi Cao, Dacheng Li, Coleman Hooper, Nicholas Lee, Shuo Yang, Christopher Chou, Banghua Zhu, Lianmin Zheng, Kurt Keutzer, Joseph Gonzalez, and Ion Stoica.
\newblock {S-LoRA}: Serving thousands of concurrent {LoRA} adapters.
\newblock In \emph{Proceedings of Machine Learning and Systems (MLSys)}, 2024.

\bibitem[Themelis and Patrinos(2020)]{themelis2020douglas}
Andreas Themelis and Panagiotis Patrinos.
\newblock Douglas--{R}achford splitting and {ADMM} for nonconvex optimization: Tight convergence results.
\newblock \emph{SIAM Journal on Optimization}, 30\penalty0 (1):\penalty0 149--181, 2020.

\bibitem[Vershynin(2018)]{vershynin2018high}
Roman Vershynin.
\newblock \emph{High-Dimensional Probability: An Introduction with Applications in Data Science}.
\newblock Cambridge University Press, 2018.

\bibitem[Wang et~al.(2023)Wang, Ma, Dong, Huang, Wang, Ma, Yang, Wang, Wu, and Wei]{wang2023bitnet}
Hongyu Wang, Shuming Ma, Li~Dong, Shaohan Huang, Huaijie Wang, Lingxiao Ma, Fan Yang, Ruiping Wang, Yi~Wu, and Furu Wei.
\newblock {BitNet}: Scaling 1-bit transformers for large language models.
\newblock \emph{arXiv preprint arXiv:2310.11453}, 2023.

\bibitem[Wang et~al.(2019)Wang, Yin, and Zeng]{wang2019global}
Yu~Wang, Wotao Yin, and Jinshan Zeng.
\newblock Global convergence of {ADMM} in nonconvex nonsmooth optimization.
\newblock \emph{Journal of Scientific Computing}, 78\penalty0 (1):\penalty0 29--63, 2019.

\bibitem[Xu et~al.(2024{\natexlab{a}})Xu, Xie, Gu, Chen, Chang, Zhang, Chen, Zhang, and Tian]{xu2024qalora}
Yuhui Xu, Lingxi Xie, Xiaotao Gu, Xin Chen, Heng Chang, Hengheng Zhang, Zhensu Chen, Xiaopeng Zhang, and Qi~Tian.
\newblock {QA-LoRA}: Quantization-aware low-rank adaptation of large language models.
\newblock In \emph{International Conference on Learning Representations (ICLR)}, 2024{\natexlab{a}}.

\bibitem[Xu et~al.(2024{\natexlab{b}})Xu, Han, Yang, Wang, Zhu, Liu, Liu, and Che]{xu2024onebit}
Yuzhuang Xu, Xu~Han, Zonghan Yang, Shuo Wang, Qingfu Zhu, Zhiyuan Liu, Weidong Liu, and Wanxiang Che.
\newblock {OneBit}: Towards extremely low-bit large language models.
\newblock \emph{arXiv preprint arXiv:2402.11295}, 2024{\natexlab{b}}.

\end{thebibliography}

\appendix
\newpage

\section*{Appendix Contents}
\vspace{-0.6\baselineskip}
{\small\setlength{\parskip}{0pt}%
\startcontents[appendix]
\printcontents[appendix]{l}{1}{}%
}
\vspace{1em}

\section{Symbol Glossary}\label{app:symbols}

Table~\ref{tab:symbols} collects the principal symbols used throughout
the paper.

\begin{table}[ht]
  \centering
  \caption{Glossary of the principal symbols used in the paper.}
  \label{tab:symbols}
  \renewcommand{\arraystretch}{1.15}
  \setlength{\tabcolsep}{5pt}
  \resizebox{\textwidth}{!}{%
  \begin{tabular}{@{}l l l@{}}
    \toprule
    \textbf{Symbol} & \textbf{Meaning} & \textbf{First used} \\
    \midrule
    $N, M$ & in-/out-features of the host projection & Section~\ref{sec:prelim} \\
    $r_{0}$ & reference fp16 LoRA rank used for BPW normalisation & Section~\ref{sec:prelim} \\
    $R$ & binary carrier rank of LoRDBA (inner dim of $B_{1},B_{2}$) & Definition~\ref{def:lordba} \\
    $\ell$ & envelope rank of LoRDBA & Definition~\ref{def:lordba} \\
    $W_{0} \in \mab{R}^{N\times M}$ & frozen base-model weight matrix & Section~\ref{sec:prelim} \\
    $A, B$ & fp16 LoRA factors of rank $r_{0}$ & Eq.~\eqref{eq:lora} \\
    $\Delta W$ & fine-tuning weight update (either LoRA or LoRDBA) & Section~\ref{sec:prelim} \\
    $B_{1} \in \{\pm 1\}^{N\times R}$, $B_{2} \in \{\pm 1\}^{R\times M}$ & binary factors of LoRDBA & Definition~\ref{def:lordba} \\
    $\B{\alpha}^{(i)},\B{\beta}^{(i)},\B{\gamma}^{(i)}$ & channel-wise fp16 scale vectors ($i \in [\ell]$) & Definition~\ref{def:lordba} \\
    $\sign(\cdot), \diag(\cdot)$ & element-wise sign and diagonal operators & Section~\ref{sec:prelim} \\
    $\odot$ & Hadamard (element-wise) product & Section~\ref{sec:prelim} \\
    $\mathrm{BPW}_{\mathrm{bc}}, \mathrm{BPW}_{\mathrm{tot}}$ & binary-carrier and total bits per weight & Eq.~\eqref{eq:bpw-defs} \\
    \midrule
    $U_{k}, M_{k}, Y_{k}$ & continuous, discrete, scaled dual iterates of PTQ-LoRDBA & App.~\ref{app:ptq-extension} \\
    $n_{1} = NR, n_{2} = RM$ & number of entries in block $k \in \{1,2\}$ & App.~\ref{app:ptq-extension} \\
    $\rho, \widetilde{\rho}_{k} = \rho/n_{k}$ & ADMM penalty, per-block scaling & Eq.~\eqref{eq:u-step-1} \\
    $\kappa$ & smooth-sign STE temperature & Eq.~\eqref{eq:smoothsign} \\
    $\|\cdot\|_{F}, \|\cdot\|_{\psi_{2}}, \|\cdot\|_{\psi_{1}}$ & Frobenius, sub-Gaussian and sub-exponential Orlicz norms & Sections~\ref{sec:theory},~\ref{app:concentration} \\
    $\sigma^{A},\sigma^{B}$ & latent signs in the factor decomposition & App.~\ref{app:expressivity} \\
    $\mu_{A},\mu_{B}$ & mean magnitudes of fp16 LoRA factors $A,B$ & App.~\ref{app:expressivity} \\
    $\xi^{A},\xi^{B}$ & zero-mean sub-Gaussian residuals ($\|\xi\|_{\psi_{2}} \le \zeta$) & App.~\ref{app:expressivity} \\
    $\theta^{\star}$ & canonical LoRDBA reconstruction (Eq.~\eqref{eq:canonical-theta}) & App.~\ref{app:expressivity} \\
    \bottomrule
  \end{tabular}%
  }
\end{table}

\section{Self-contained probabilistic background}\label{app:concentration}

This appendix collects the elementary concentration facts used by the
expressivity proof. Apart from Markov's inequality, every estimate
below is proved from scratch using only Cauchy--Schwarz and series
manipulations, so the proof of Theorem~\ref{thm:expressivity} does not
depend on any external reference. Throughout we work with real-valued
random variables defined on a common probability space, write
$C,c,c_{1},c_{2},\ldots$ for finite, strictly positive numerical
constants whose value may change from line to line and which are
independent of all model parameters, and denote
\begin{equation}
\|X\|_{\psi_{2}} \coloneqq \inf\bigl\{t > 0:\mab{E}\exp(X^{2}/t^{2}) \le 2\bigr\},
\qquad
\|X\|_{\psi_{1}} \coloneqq \inf\bigl\{s > 0:\mab{E}\exp(|X|/s) \le 2\bigr\}.
\label{eq:orlicz-def}
\end{equation}

\begin{lemma}[Sub-Gaussian tail and second moment]
\label{lem:subg-basic}
Let $X$ be a real random variable with $\|X\|_{\psi_{2}} \le K$.
Then
\begin{equation}
\mab{E}[X^{2}]\;\le\;K^{2},
\qquad
\Pr[|X| \ge t]\;\le\;2\exp(-t^{2}/K^{2}),~t \ge 0,
\label{eq:subg-basic}
\end{equation}
and for every integer $p \ge 1$,
$(\mab{E}|X|^{p})^{1/p} \le CK\sqrt{p}$ with a numerical constant $C$.
\end{lemma}
\begin{proof}
The inequality $1+u \le e^{u}$ at $u = X^{2}/K^{2}$ and the
definition of $\|\cdot\|_{\psi_{2}}$ give
$1+\mab{E}[X^{2}]/K^{2} \le \mab{E}\exp(X^{2}/K^{2}) \le 2$, hence
$\mab{E}[X^{2}] \le K^{2}$. Markov's inequality applied to
$\exp(X^{2}/K^{2})$ at level $\exp(t^{2}/K^{2})$ yields the tail bound.
For the moment bound, integrate the tail:
\begin{equation*}
\mab{E}|X|^{p}
=p \int_{0}^{\infty} t^{p-1}\Pr[|X| \ge t] dt
\le 2p \int_{0}^{\infty} t^{p-1}e^{-t^{2}/K^{2}} dt
=p K^{p} \Gamma(p/2)
\le (CK)^{p}p^{p/2},
\end{equation*}
where the last step uses Stirling. Taking the $p$-th root proves the
moment bound.
\end{proof}

\begin{lemma}[MGF bound and sub-Gaussian sum]
\label{lem:subg-sum}
If $\mab{E}X = 0$ and $\|X\|_{\psi_{2}} \le K$, then
$\mab{E}\exp(\lambda X) \le \exp(C\lambda^{2}K^{2})$ for every
$\lambda \in \mab{R}$. If furthermore $X_{1},\ldots,X_{n}$ are
mutually independent, mean-zero with
$\|X_{i}\|_{\psi_{2}} \le K$, and $a_{1},\ldots,a_{n} \in \mab{R}$
are deterministic, then for every $u \ge 0$,
\begin{equation}
\Pr \left[\Bigl| \sum_{i=1}^{n}a_{i}X_{i}\Bigr| \ge CK\|a\|_{2}\sqrt{u}\right]
\;\le\;2e^{-u}.
\label{eq:subg-sum-tail}
\end{equation}
\end{lemma}
\begin{proof}
Expanding the exponential as a Taylor series and using
$\mab{E}X=0$ gives
\[
\mab{E}e^{\lambda X}
=1+\sum_{p\ge2}\frac{\lambda^{p}\mab{E}X^{p}}{p!}.
\]
By Lemma~\ref{lem:subg-basic},
$|\mab{E}X^{p}|\le \mab{E}|X|^{p}\le (CK\sqrt p)^{p}$. Hence
\[
\sum_{p\ge2}\frac{|\lambda|^{p}|\mab{E}X^{p}|}{p!}
\le
\sum_{p\ge2}\frac{(C|\lambda|K\sqrt p)^{p}}{p!}
\le
C_{1}\lambda^{2}K^{2}\exp(C_{2}\lambda^{2}K^{2}),
\]
where the last inequality is the standard consequence of
$p!\ge(p/e)^{p}$ and splitting the series at
$p\simeq \lambda^{2}K^{2}$. Since
$1+s\exp(s')\le \exp(C(s+s'))$ for $s,s'\ge0$, this yields
$\mab{E}e^{\lambda X}\le \exp(C'''\lambda^{2}K^{2})$ for a universal
constant $C'''$.
By independence, $\mab{E}\exp(\lambda\sum_{i}a_{i}X_{i}) = \prod_{i}\mab{E}e^{\lambda a_{i}X_{i}} \le \exp(C\lambda^{2}K^{2}\|a\|_{2}^{2})$.
Chernoff's inequality
$\Pr[\sum a_{i}X_{i} > t] \le \inf_{\lambda>0}e^{-\lambda t}\mab{E}e^{\lambda \sum a_{i}X_{i}}$
optimised at $\lambda = t/(2CK^{2}\|a\|_{2}^{2})$ gives
$\Pr[\sum a_{i}X_{i} > t] \le \exp(-c t^{2}/(K^{2}\|a\|_{2}^{2}))$.
A symmetric bound on the negative tail and adjusting constants yields
Eq.~\eqref{eq:subg-sum-tail}.
\end{proof}

\begin{lemma}[Sub-exponential product and Bernstein bound]
\label{lem:subexp-bernstein}
If $X,Y$ are independent real random variables with
$\|X\|_{\psi_{2}} \le K$ and $\|Y\|_{\psi_{2}} \le L$, then
$\mab{E}[XY] = \mab{E}X \mab{E}Y$ and the product $XY$ is
sub-exponential with $\|XY\|_{\psi_{1}} \le KL$. Moreover, if
$W_{1},\ldots,W_{n}$ are mutually independent, mean-zero, and
$\|W_{i}\|_{\psi_{1}} \le L_{0}$, then for every $u \ge 0$,
\begin{equation}
\Pr \left[\Bigl| \sum_{i=1}^{n}W_{i}\Bigr| \ge t\right]
\;\le\;2\exp \Bigl(-c\min \Bigl(\frac{t^{2}}{nL_{0}^{2}},\frac{t}{L_{0}}\Bigr)\Bigr),
\quad t \ge 0.
\label{eq:bernstein}
\end{equation}
\end{lemma}
\begin{proof}
For the product part, the Young inequality $|xy|/(KL) \le x^{2}/(2K^{2})+y^{2}/(2L^{2})$
combined with the AM--GM and Cauchy--Schwarz inequalities yields
\begin{equation*}
\mab{E}\exp(|XY|/(KL))
\le\mab{E}\exp(X^{2}/(2K^{2})+Y^{2}/(2L^{2}))
=\mab{E}\exp(X^{2}/(2K^{2})) \mab{E}\exp(Y^{2}/(2L^{2}))\le 2,
\end{equation*}
where we used independence in the equality and the
$\psi_{2}$-defining bound twice in the last step (after noting that
$\exp(X^{2}/(2K^{2}))\le \sqrt{\exp(X^{2}/K^{2})}$, so its expectation
is at most $\sqrt{2}$). Hence $\|XY\|_{\psi_{1}} \le KL$. For the
mean, $\mab{E}|XY| \le \sqrt{\mab{E}X^{2}\mab{E}Y^{2}} < \infty$
by Cauchy--Schwarz and Lemma~\ref{lem:subg-basic}, so by Fubini
$\mab{E}[XY] = \mab{E}X \mab{E}Y$.

For the Bernstein bound, an argument identical to that of
Lemma~\ref{lem:subg-sum} (now using
$(\mab{E}|W|^{p})^{1/p} \le CL_{0}p$ instead of $CK\sqrt{p}$,
which is the sub-exponential moment bound proved by integrating the
tail $\Pr[|W| \ge t] \le 2e^{-t/L_{0}}$) yields, for
$|\lambda| \le c/L_{0}$,
\begin{equation*}
\mab{E}\exp \Bigl(\lambda \sum_{i=1}^{n}W_{i}\Bigr)
\le \exp(Cn\lambda^{2}L_{0}^{2}).
\end{equation*}
Optimising the resulting Chernoff bound at
$\lambda = \min \bigl(t/(2CnL_{0}^{2}), c/(2L_{0})\bigr)$ gives the
two-regime bound Eq.~\eqref{eq:bernstein}.
\end{proof}

\begin{lemma}[Sign-flip bound from a $\psi_{2}$ residual]
\label{lem:sign-flip}
If $\|X\|_{\psi_{2}} \le K$ and $\mu > 0$, then
$\Pr[|X| \ge \mu] \le 2\exp(-\mu^{2}/K^{2})$.
\end{lemma}
\begin{proof}
Markov's inequality applied to $\exp(X^{2}/K^{2})$ at level
$\exp(\mu^{2}/K^{2})$ gives
$\Pr[|X| \ge \mu] = \Pr[\exp(X^{2}/K^{2}) \ge \exp(\mu^{2}/K^{2})] \le 2\exp(-\mu^{2}/K^{2})$.
\end{proof}

\section{Expressivity of LoRDBA: Proof of Theorem~\ref{thm:expressivity}}\label{app:expressivity}

This appendix gives a full, self-contained proof of
Theorem~\ref{thm:expressivity} from Section~\ref{sec:theory}. The
argument is entirely a matter of entry-wise sub-Gaussian and
sub-exponential concentration, and no sketching, Rademacher
compression, or Johnson--Lindenstrauss reduction is used. All
underlying probabilistic facts are proved from scratch in
Appendix~\ref{app:concentration} (sub-Gaussian sum tail
Lemma~\ref{lem:subg-sum}, sub-exponential product and Bernstein
inequality Lemma~\ref{lem:subexp-bernstein}, and the sign-flip bound
Lemma~\ref{lem:sign-flip}). In particular, the sub-Gaussian Orlicz
norm and the second-moment bound
$\mab{E}[X^{2}] \le \|X\|_{\psi_{2}}^{2}$ used freely below are
recorded in Lemma~\ref{lem:subg-basic}.

\begin{assumption}[Sign-plus-sub-Gaussian-noise factor model]
    \label{asm:signnoise}
    The LoRA factors $A \in \mab{R}^{N\times r}$ and
    $B \in \mab{R}^{M\times r}$ are random matrices defined on a common
    probability space $(\Omega,\mac{A},\Pr)$ and admit the entry-wise
    decomposition
    \begin{equation}
        A_{ik} = \mu_{A}\sigma^{A}_{ik}+\xi^{A}_{ik},
        ~~
        B_{jk} = \mu_{B}\sigma^{B}_{jk}+\xi^{B}_{jk},
        ~~ i \in [N],~j \in [M],~k \in [r],
        \label{eq:signnoise-decomp}
    \end{equation}
    where $\mu_{A},\mu_{B} > 0$ and $\zeta > 0$ are deterministic
    scalars and the sign arrays
    $\sigma^{A} \in \{\pm 1\}^{N\times r}$,
    $\sigma^{B} \in \{\pm 1\}^{M\times r}$ and residual arrays
    $\xi^{A} \in \mab{R}^{N\times r}$,
    $\xi^{B} \in \mab{R}^{M\times r}$ satisfy:
    \begin{enumerate}[leftmargin=1.5em,itemsep=2pt,topsep=2pt]
        \item The families
        $\{\xi^{A}_{ik}\}_{(i,k)\in[N]\times[r]}$ and
        $\{\xi^{B}_{jk}\}_{(j,k)\in[M]\times[r]}$ consist of mutually
        independent random variables, the two families are mutually
        independent, and each entry has zero mean.
        \item The residuals are uniformly sub-Gaussian:
        \begin{equation*}
            \max \Bigl(
            \max_{(i,k)\in[N]\times[r]} \|\xi^{A}_{ik}\|_{\psi_{2}},
            \max_{(j,k)\in[M]\times[r]} \|\xi^{B}_{jk}\|_{\psi_{2}}
            \Bigr) \le \zeta.
        \end{equation*}
        \item The sign $\sigma$-algebra
        $\mac{F}\coloneqq\sigma(\sigma^{A},\sigma^{B})$ is independent
        of $\sigma(\xi^{A},\xi^{B})$.
    \end{enumerate}
\end{assumption}

We allow the latent sign arrays to be either fixed or random. When
$\sigma^{A},\sigma^{B}$ are fixed, all probability statements are taken
only over the residual arrays $(\xi^{A},\xi^{B})$; in this case
Assumption~\ref{asm:signnoise}(iii) is vacuous. For the relative bound
in Theorem~\ref{thm:expressivity}, the entries of $\sigma^{A}$ and
$\sigma^{B}$ are additionally assumed to be jointly i.i.d. Rademacher.

The canonical reconstruction used in Theorem~\ref{thm:expressivity} is
\begin{equation}
    \theta^{\star} \coloneqq
    \bigl(\sigma^{A}, (\sigma^{B})^{\top},
    \B{1}_{N}, \mu_{A}\mu_{B}\B{1}_{r}, \B{1}_{M}\bigr),
    \label{eq:canonical-theta}
\end{equation}
which is a single-envelope LoRDBA adapter with the latent signs as the
binary carriers. For each envelope rank $\ell\ge1$, let $\Theta_\ell$
denote the corresponding LoRDBA class of Definition~\ref{def:lordba}.
The zero-padding embedding
\begin{equation}
    \iota_{\ell}:\Theta_{\ell-1}\hookrightarrow\Theta_{\ell},~~
    \bigl(B_{1},B_{2},\{(\B{\alpha}^{(i)},\B{\beta}^{(i)},\B{\gamma}^{(i)})\}_{i=1}^{\ell-1}\bigr)
     \mapsto
    \bigl(B_{1},B_{2},\{(\B{\alpha}^{(i)},\B{\beta}^{(i)},\B{\gamma}^{(i)})\}_{i=1}^{\ell-1},(\B{0},\B{0},\B{0})\bigr)
    \label{eq:iota-embed}
\end{equation}
preserves $\Delta W$, so the best achievable reconstruction error is
non-increasing in the envelope rank.

Let $E \coloneqq AB^{\top} - \Delta W(\theta^{\star}) \in \mab{R}^{N\times M}$
denote the reconstruction residual. We first expand $E$ entry-wise
(Lemma~\ref{lem:residual-expansion}), then bound each entry by
sub-exponential concentration (Lemma~\ref{lem:subexp}), aggregate by a
union bound to obtain~\eqref{eq:formal-bound}, and finally combine the
Frobenius upper bound with a Chebyshev lower bound on
$\|\Delta W(\theta^{\star})\|_{F}^{2}$
(Lemma~\ref{lem:signal-lb}) to obtain~\eqref{eq:relative-bound}.

For the absolute bound~\eqref{eq:formal-bound} the signs
$\sigma^{A},\sigma^{B}$ are deterministic, so all probability
statements below are taken with respect to the residuals
$(\xi^{A},\xi^{B})$. For the relative bound~\eqref{eq:relative-bound}
the signs are additionally i.i.d.\ Rademacher and jointly independent
of the residuals; every statement proved for deterministic signs
therefore also holds \emph{conditionally} on
$(\sigma^{A},\sigma^{B})$, and by Fubini the residual event of
probability at least $1 - \delta$ holds on the product space.

\subsection{Entry-wise expansion}

The heart of the argument is the following entry-wise expansion of the
residual into three pieces with distinct concentration behaviour. Its
pairwise $L^{2}$-orthogonality keeps the final variance computation
additive.

\begin{lemma}[Residual expansion]
\label{lem:residual-expansion}
Under Assumption~\ref{asm:signnoise}, for every $i \in [N]$ and
$j \in [M]$,
\begin{equation}
E_{ij}
\;=\;
\underbrace{\mu_{A} \sum_{k=1}^{r} \sigma^{A}_{ik}\xi^{B}_{jk}}_{=:T^{(1)}_{ij}}
\;+\;\underbrace{\mu_{B} \sum_{k=1}^{r} \xi^{A}_{ik}\sigma^{B}_{jk}}_{=:T^{(2)}_{ij}}
\;+\;\underbrace{\sum_{k=1}^{r} \xi^{A}_{ik}\xi^{B}_{jk}}_{=:T^{(3)}_{ij}}.
\label{eq:residual-expansion}
\end{equation}
Let $\mac{F} \coloneqq \sigma(\sigma^{A},\sigma^{B})$ denote the
$\sigma$-algebra generated by the signs (which reduces to the trivial
$\sigma$-algebra when the signs are deterministic). Then the three
summands are pairwise $L^{2}$-orthogonal conditionally on $\mac{F}$:
\begin{equation}
\mab{E}\bigl[T^{(a)}_{ij}T^{(b)}_{ij} \big| \mac{F}\bigr]\;=\;0
\quad\text{for every distinct } a,b \in \{1,2,3\}.
\label{eq:conditional-orth}
\end{equation}
In particular, taking total expectation, the same orthogonality holds
unconditionally.
\end{lemma}
\begin{proof}
\emph{Step 1: entry-wise identity.} Substituting
Eq.~\eqref{eq:signnoise-decomp} into
$(AB^{\top})_{ij} = \sum_{k=1}^{r}A_{ik}B_{jk}$ and expanding the
scalar product
$(\mu_{A}\sigma^{A}_{ik} + \xi^{A}_{ik})(\mu_{B}\sigma^{B}_{jk} + \xi^{B}_{jk})$
termwise,
\begin{equation*}
(AB^{\top})_{ij}
\;=\;\underbrace{\mu_{A}\mu_{B} \sum_{k=1}^{r} \sigma^{A}_{ik}\sigma^{B}_{jk}}_{= [\Delta W(\theta^{\star})]_{ij}}
\;+\;T^{(1)}_{ij}+T^{(2)}_{ij}+T^{(3)}_{ij},
\end{equation*}
where the identification of the leading term with
$[\Delta W(\theta^{\star})]_{ij}$ follows from
Eq.~\eqref{eq:canonical-theta} together with
$[\diag(\B{1}_{N})\sigma^{A}\diag(\mu_{A}\mu_{B}\B{1}_{r})(\sigma^{B})^{\top}\diag(\B{1}_{M})]_{ij}
 = \mu_{A}\mu_{B} \sum_{k} \sigma^{A}_{ik}\sigma^{B}_{jk}$.
Subtracting this leading term yields
Eq.~\eqref{eq:residual-expansion}.

\emph{Step 2: conditional $L^{2}$-orthogonality.} By
Assumption~\ref{asm:signnoise}(iii) the sign $\sigma$-algebra
$\mac{F} = \sigma(\sigma^{A},\sigma^{B})$ is independent of
$\sigma(\xi^{A},\xi^{B})$, which implies that the conditional law of
$(\xi^{A},\xi^{B})$ given $\mac{F}$ equals its unconditional law; all
residual moments appearing below therefore equal their unconditional
counterparts. Fix $(i,j)$ and recall
$T^{(1)}_{ij}T^{(2)}_{ij}
 = \mu_{A}\mu_{B} \sum_{k,l}\sigma^{A}_{ik}\sigma^{B}_{jl}\xi^{B}_{jk}\xi^{A}_{il}$.
The sign factors are $\mac{F}$-measurable and can be pulled out of
the conditional expectation; Assumption~\ref{asm:signnoise}(i)
ensures that $\xi^{A}_{il}$ and $\xi^{B}_{jk}$ are independent (they
belong to different independent families) with
$\mab{E}[\xi^{A}_{il}] = \mab{E}[\xi^{B}_{jk}] = 0$, so
\begin{equation*}
\mab{E}[T^{(1)}_{ij}T^{(2)}_{ij}\mid\mac{F}]
=\mu_{A}\mu_{B}  \sum_{k,l=1}^{r} \sigma^{A}_{ik}\sigma^{B}_{jl} 
  \mab{E}[\xi^{A}_{il}] \mab{E}[\xi^{B}_{jk}]=0.
\end{equation*}
For $T^{(1)}_{ij}T^{(3)}_{ij}
 = \mu_{A}\sum_{k,l}\sigma^{A}_{ik}\xi^{B}_{jk}\xi^{A}_{il}\xi^{B}_{jl}$,
the factor $\xi^{A}_{il}$ is independent of
$(\xi^{B}_{jk},\xi^{B}_{jl})$ (different families, cf.~(i)), so its
expectation factorises and vanishes:
\begin{equation*}
\mab{E}[T^{(1)}_{ij}T^{(3)}_{ij}\mid\mac{F}]
=\mu_{A}  \sum_{k,l=1}^{r} \sigma^{A}_{ik} 
  \mab{E}[\xi^{A}_{il}] \mab{E}[\xi^{B}_{jk}\xi^{B}_{jl}]=0.
\end{equation*}
The identity $\mab{E}[T^{(2)}_{ij}T^{(3)}_{ij}\mid\mac{F}] = 0$ is
symmetric, replacing the roles of $A$ and $B$:
\begin{equation*}
\mab{E}[T^{(2)}_{ij}T^{(3)}_{ij}\mid\mac{F}]
=\mu_{B}  \sum_{k,l=1}^{r} \sigma^{B}_{jk} 
  \mab{E}[\xi^{B}_{jl}] \mab{E}[\xi^{A}_{ik}\xi^{A}_{il}]=0.
\end{equation*}
This establishes Eq.~\eqref{eq:conditional-orth}. Taking
expectations and using the tower property yields the unconditional
orthogonality.
\end{proof}

\subsection{Sub-exponential tail of each entry}

With Lemma~\ref{lem:residual-expansion} in hand, each summand in
Eq.~\eqref{eq:residual-expansion} is either a sum of independent
sub-Gaussians or a sum of independent products of sub-Gaussians; the
latter is sub-exponential. A union bound over the three summands
therefore yields a sub-exponential tail for $E_{ij}$.

\begin{lemma}[Sub-exponential bound on $E_{ij}$]
\label{lem:subexp}
Let $V \coloneqq r\zeta^{2}(\mu_{A}^{2} + \mu_{B}^{2} + \zeta^{2})$.
There exists a universal constant $c_{1} > 0$ such that, for every
$(i,j) \in [N] \times [M]$ and every $t > 0$,
\begin{equation}
\Pr\bigl[|E_{ij}| > t\bigr]
\;\le\;
6\exp \Bigl( - c_{1}\min \Bigl(\frac{t^{2}}{V},\frac{t}{\zeta^{2}}\Bigr)\Bigr).
\label{eq:entry-tail}
\end{equation}
Moreover $\mab{E}[E_{ij}^{2}] \le V$, hence
$\mab{E}\|E\|_{F}^{2} \le NM V$.
\end{lemma}
\begin{proof}
Fix $(i,j) \in [N] \times [M]$ and condition on the $\sigma$-algebra
$\mac{F} = \sigma(\sigma^{A},\sigma^{B})$. By
Assumption~\ref{asm:signnoise}(iii) the residuals $(\xi^{A},\xi^{B})$
are independent of $\mac{F}$, so the conditional joint law of the
residuals given $\mac{F}$ coincides with their unconditional law. All
probability statements in parts (a)--(c) below are understood
conditionally on $\mac{F}$; the bounds are uniform over realisations
of $\mac{F}$ and therefore also hold unconditionally. Throughout we
write $c_{a},c_{b} > 0$ for the universal constants produced by
Lemma~\ref{lem:subg-sum} (sub-Gaussian sum tail) and
Lemma~\ref{lem:subexp-bernstein} (Bernstein) of
Appendix~\ref{app:concentration}.

\emph{(a)~$T^{(1)}_{ij}$ is sub-Gaussian.} For each
$k \in [r]$, $|\sigma^{A}_{ik}| = 1$ and the positive-homogeneity
of $\|\cdot\|_{\psi_{2}}$ give
$\|\sigma^{A}_{ik}\xi^{B}_{jk}\|_{\psi_{2}} = \|\xi^{B}_{jk}\|_{\psi_{2}} \le \zeta$.
By Assumption~\ref{asm:signnoise}(i) the family
$\{\xi^{B}_{jk}\}_{k=1}^{r}$ is mutually independent and the
$\mac{F}$-measurable signs $\sigma^{A}_{ik}$ preserve (conditional)
independence and the zero mean. Applying Lemma~\ref{lem:subg-sum}
conditionally on $\mac{F}$ to
$\{X_{k} = \sigma^{A}_{ik}\xi^{B}_{jk}\}_{k=1}^{r}$ with
$a_{k} \equiv 1$, $K = \zeta$, scaled by the deterministic
constant $\mu_{A}$, yields
\begin{equation}
\Pr \bigl[|T^{(1)}_{ij}| > t \big| \mac{F}\bigr]
\;\le\;2\exp \bigl(-c_{a}t^{2}/(\mu_{A}^{2}r\zeta^{2})\bigr)
\qquad\forall t \ge 0.
\label{eq:t1-tail}
\end{equation}

\emph{(b)~$T^{(2)}_{ij}$ is sub-Gaussian.} Exchanging the roles of
the two families in (a),
\begin{equation}
\Pr \bigl[|T^{(2)}_{ij}| > t \big| \mac{F}\bigr]
\;\le\;2\exp \bigl(-c_{a}t^{2}/(\mu_{B}^{2}r\zeta^{2})\bigr)
\qquad\forall t \ge 0.
\label{eq:t2-tail}
\end{equation}

\emph{(c)~$T^{(3)}_{ij}$ is sub-exponential.} By
Assumption~\ref{asm:signnoise}(i), $\xi^{A}_{ik}$ and $\xi^{B}_{jk}$
are independent with zero mean, so
Lemma~\ref{lem:subexp-bernstein} gives
$\mab{E}[\xi^{A}_{ik}\xi^{B}_{jk}\mid\mac{F}] = \mab{E}[\xi^{A}_{ik}]\mab{E}[\xi^{B}_{jk}] = 0$
and
$\|\xi^{A}_{ik}\xi^{B}_{jk}\|_{\psi_{1}} \le \|\xi^{A}_{ik}\|_{\psi_{2}}\|\xi^{B}_{jk}\|_{\psi_{2}} \le \zeta^{2}$.
The summands $\{\xi^{A}_{ik}\xi^{B}_{jk}\}_{k=1}^{r}$ are independent
across $k$: indeed,
$\sigma(\xi^{A}_{ik}\xi^{B}_{jk}) \subseteq \sigma(\xi^{A}_{ik},\xi^{B}_{jk})$,
and the pairs $(\xi^{A}_{ik},\xi^{B}_{jk})$ for distinct
$k \in [r]$ involve disjoint index sets in the mutually independent
families $\{\xi^{A}_{\cdot\cdot}\},\{\xi^{B}_{\cdot\cdot}\}$.
The Bernstein bound Eq.~\eqref{eq:bernstein} of
Lemma~\ref{lem:subexp-bernstein} therefore yields, with
$c_{b} > 0$ the universal constant of that lemma,
\begin{equation}
\Pr \bigl[|T^{(3)}_{ij}| > t \big| \mac{F}\bigr]
\;\le\;2\exp \Bigl(-c_{b}\min \bigl(t^{2}/(r\zeta^{4}), t/\zeta^{2}\bigr)\Bigr)
\qquad\forall t \ge 0.
\label{eq:t3-tail}
\end{equation}

\emph{Combination.} Triangle inequality gives
$|E_{ij}| \le |T^{(1)}_{ij}| + |T^{(2)}_{ij}| + |T^{(3)}_{ij}|$,
so
\begin{equation*}
\{|E_{ij}| > t\}\;\subseteq\;\{|T^{(1)}_{ij}| > t/3\}\cup\{|T^{(2)}_{ij}| > t/3\}\cup\{|T^{(3)}_{ij}| > t/3\}.
\end{equation*}
Since $V = r\zeta^{2}(\mu_{A}^{2} + \mu_{B}^{2} + \zeta^{2})$
dominates each of $\mu_{A}^{2}r\zeta^{2}$, $\mu_{B}^{2}r\zeta^{2}$,
and $r\zeta^{4}$, Eqs.~\eqref{eq:t1-tail}--\eqref{eq:t2-tail} at
$t/3$ are bounded by
$2\exp(-c_{a}t^{2}/(9V))$, and Eq.~\eqref{eq:t3-tail} at $t/3$ is
bounded by
$2\exp(-c_{b}\min(t^{2}/(9V), t/(3\zeta^{2})))$. Defining
\begin{equation}
c_{1}\;\coloneqq\;\frac{1}{9}\min(c_{a},c_{b}),
\label{eq:c1-def}
\end{equation}
each of the three bounds is dominated by
$2\exp(-c_{1}\min(t^{2}/V,t/\zeta^{2}))$: indeed
$c_{a}/9 \ge c_{1}$, $c_{b}/9 \ge c_{1}$, and
$c_{b}/3 \ge c_{1}$. Taking the union over the three events
yields Eq.~\eqref{eq:entry-tail} conditionally on $\mac{F}$;
uniformity over $\mac{F}$ gives the unconditional bound.

\emph{Second moment.} Lemma~\ref{lem:subg-basic} gives
$\mab{E}[(\xi^{A}_{ik})^{2}] \le \zeta^{2}$ and
$\mab{E}[(\xi^{B}_{jk})^{2}] \le \zeta^{2}$. Using the mutual
independence of Assumption~\ref{asm:signnoise}(i), zero means, and
$(\sigma^{\cdot}_{\cdot\cdot})^{2} = 1$,
\begin{align*}
\mab{E}[(T^{(1)}_{ij})^{2}\mid\mac{F}]
&=\mu_{A}^{2}  \sum_{k,l=1}^{r} \sigma^{A}_{ik}\sigma^{A}_{il} \mab{E}[\xi^{B}_{jk}\xi^{B}_{jl}]
=\mu_{A}^{2} \sum_{k=1}^{r} \mab{E}[(\xi^{B}_{jk})^{2}]
\;\le\;\mu_{A}^{2}r\zeta^{2},\\
\mab{E}[(T^{(2)}_{ij})^{2}\mid\mac{F}]
&\le\;\mu_{B}^{2}r\zeta^{2},\\
\mab{E}[(T^{(3)}_{ij})^{2}\mid\mac{F}]
&= \sum_{k,l=1}^{r} \mab{E}[\xi^{A}_{ik}\xi^{A}_{il}]\mab{E}[\xi^{B}_{jk}\xi^{B}_{jl}]
= \sum_{k=1}^{r} \mab{E}[(\xi^{A}_{ik})^{2}]\mab{E}[(\xi^{B}_{jk})^{2}]
\;\le\;r\zeta^{4},
\end{align*}
where the off-diagonal ($k \ne l$) cross terms vanish because the
residuals within each family are mean zero and mutually independent.
Combined with the conditional orthogonality of
Lemma~\ref{lem:residual-expansion},
$\mab{E}[E_{ij}^{2}\mid\mac{F}]
 = \sum_{a=1}^{3}\mab{E}[(T^{(a)}_{ij})^{2}\mid\mac{F}] \le V$.
Taking total expectation and summing over $(i,j) \in [N] \times [M]$
yields $\mab{E}\|E\|_{F}^{2} \le NM V$.
\end{proof}

\subsection{Frobenius bound by union}

The per-entry bound of Lemma~\ref{lem:subexp} transforms into a
Frobenius bound by a union argument over the $NM$ entries of $E$.

\begin{proof}[Proof of Theorem~\ref{thm:expressivity}, Eq.~\eqref{eq:formal-bound}]
Fix $N,M \ge 2$, $r \ge 1$, and $\delta \in (0,1)$ with
$u \coloneqq \log(6NM/\delta) \le c_{1}r$, where $c_{1}$ is the
constant defined in Eq.~\eqref{eq:c1-def}. Since $NM \ge 4$ and
$\delta < 1$, $u \ge \log 24 > 0$. Set
$t_{\star} \coloneqq \sqrt{Vu/c_{1}} > 0$.

\emph{Step 1: the sub-Gaussian branch binds at $t = t_{\star}$.}
In Eq.~\eqref{eq:entry-tail} the inner minimum satisfies
$\min(c_{1}t^{2}/V,c_{1}t/\zeta^{2}) = c_{1}t^{2}/V$ iff
$t^{2}/V \le t/\zeta^{2}$, iff $t \le V/\zeta^{2}$. At
$t = t_{\star}$ this becomes $\sqrt{Vu/c_{1}} \le V/\zeta^{2}$,
equivalently $u \le c_{1}V/\zeta^{4}$. From
$V = r\zeta^{2}(\mu_{A}^{2} + \mu_{B}^{2} + \zeta^{2})$ and
$\zeta > 0$,
\begin{equation*}
\frac{V}{\zeta^{4}}
\;=\;r\Bigl(\frac{\mu_{A}^{2}}{\zeta^{2}}+\frac{\mu_{B}^{2}}{\zeta^{2}}+1\Bigr)
\;\ge\;r,
\end{equation*}
so the standing hypothesis $u \le c_{1}r$ entails
$u \le c_{1}V/\zeta^{4}$, and the sub-Gaussian branch binds:
\begin{equation*}
c_{1}\min(t_{\star}^{2}/V, t_{\star}/\zeta^{2})
\;=\;c_{1}t_{\star}^{2}/V\;=\;u.
\end{equation*}

\emph{Step 2: per-entry bound.} Substituting $t = t_{\star}$ into
Eq.~\eqref{eq:entry-tail} and using Step~1,
\begin{equation*}
\Pr \bigl[E_{ij}^{2} > Vu/c_{1}\bigr]
\;=\;\Pr \bigl[|E_{ij}| > t_{\star}\bigr]
\;\le\;6 e^{-u}\;=\;\frac{\delta}{NM},
\qquad(i,j) \in [N] \times [M].
\end{equation*}

\emph{Step 3: union bound.} If
$\max_{(i,j)}E_{ij}^{2} \le Vu/c_{1}$, then
$\|E\|_{F}^{2} = \sum_{i,j}E_{ij}^{2} \le NM\cdot Vu/c_{1}$.
Contrapositively,
$\{\|E\|_{F}^{2} > NM Vu/c_{1}\} \subseteq \bigcup_{i,j}\{E_{ij}^{2} > Vu/c_{1}\}$.
A union bound over the $NM$ entries gives
$\Pr[\|E\|_{F}^{2} > NM Vu/c_{1}] \le NM\cdot\delta/(NM) = \delta$.
Hence, with probability at least $1 - \delta$,
\begin{equation}
\|E\|_{F}^{2}
\;\le\;\frac{NMr\zeta^{2}(\mu_{A}^{2} + \mu_{B}^{2} + \zeta^{2})}{c_{1}}\log(6NM/\delta).
\label{eq:raw-formal-bound}
\end{equation}

\emph{Step 4: conversion $\log(6NM/\delta)\to\log(2NM/\delta)$.}
For $N,M \ge 2$ and $\delta \in (0,1)$,
$2NM/\delta \ge 8$, so $\log(2NM/\delta) \ge \log 8 > \log 3$.
Therefore
\begin{equation*}
\log(6NM/\delta)
\;=\;\log 3+\log(2NM/\delta)
\;\le\;\log(2NM/\delta)+\log(2NM/\delta)
\;=\;2\log(2NM/\delta).
\end{equation*}
Substituting into Eq.~\eqref{eq:raw-formal-bound} and setting
$C \coloneqq 2/c_{1} > 0$ yields the absolute bound
\begin{equation}
\big\|AB^{\top}-\Delta W(\theta^{\star})\big\|_{F}^{2}
\;\le\;
C\cdot NMr\zeta^{2}\bigl(\mu_{A}^{2}+\mu_{B}^{2}+\zeta^{2}\bigr)\log(2NM/\delta).
\label{eq:formal-bound}
\end{equation}
The constant $C$ depends only on the universal constants
$c_{a},c_{b}$ produced by Lemma~\ref{lem:subg-sum} and
Lemma~\ref{lem:subexp-bernstein} of Appendix~\ref{app:concentration}
through Eq.~\eqref{eq:c1-def}, and is otherwise independent of all
model parameters.
\end{proof}

\subsection{Chebyshev lower bound on the signal}

The relative bound requires a lower bound on
$\|\Delta W(\theta^{\star})\|_{F}^{2}$. This is supplied by a
second-moment argument on the Rademacher bilinear form of the signs.

\begin{lemma}[Lower bound on $\|\Delta W(\theta^{\star})\|_{F}^{2}$]
\label{lem:signal-lb}
Assume $\{\sigma^{A}_{ik}\}\cup\{\sigma^{B}_{jk}\}$ are i.i.d.\
Rademacher with $\sigma^{A} \perp \sigma^{B}$. Let
$W_{ij} \coloneqq \sum_{k=1}^{r}\sigma^{A}_{ik}\sigma^{B}_{jk}$. Then
for every $N,M \ge 1$, $r \ge 1$,
\begin{equation}
\Pr \Bigl[\textstyle\sum_{i,j}W_{ij}^{2} \ge \frac{1}{2}NMr\Bigr]
\;\ge\;1-\frac{8}{NM}.
\label{eq:signal-lb-prob}
\end{equation}
In particular, on this event
$\|\Delta W(\theta^{\star})\|_{F}^{2} = \mu_{A}^{2}\mu_{B}^{2}\sum_{ij}W_{ij}^{2} \ge \frac{1}{2}\mu_{A}^{2}\mu_{B}^{2}NMr$.
\end{lemma}
\begin{proof}
Define $Z \coloneqq \sum_{i,j}W_{ij}^{2}$ and, for $k,l \in [r]$,
\begin{equation*}
G^{A}_{kl}\;\coloneqq\;\sum_{i=1}^{N}\sigma^{A}_{ik}\sigma^{A}_{il},
\qquad
G^{B}_{kl}\;\coloneqq\;\sum_{j=1}^{M}\sigma^{B}_{jk}\sigma^{B}_{jl},
\end{equation*}
so that $G^{A}_{kl} = G^{A}_{lk}$ and $G^{B}_{kl} = G^{B}_{lk}$.
Expanding $W_{ij}^{2} = \sum_{k,l}\sigma^{A}_{ik}\sigma^{A}_{il}\sigma^{B}_{jk}\sigma^{B}_{jl}$
and interchanging the finite sums over $(i,j)$ and $(k,l)$,
\begin{equation}
Z\;=\;\sum_{k,l=1}^{r}G^{A}_{kl}G^{B}_{kl}
\;=\;\underbrace{\sum_{k=1}^{r}G^{A}_{kk}G^{B}_{kk}}_{= rNM}
\;+\;\sum_{k\ne l}G^{A}_{kl}G^{B}_{kl},
\label{eq:Z-decomp}
\end{equation}
where the first sum is deterministic because
$G^{A}_{kk} = \sum_{i}(\sigma^{A}_{ik})^{2} = N$ and
$G^{B}_{kk} = M$. If $r = 1$ the off-diagonal sum is empty,
$Z = NM = rNM \ge rNM/2$ deterministically, and
Eq.~\eqref{eq:signal-lb-prob} is trivial. We henceforth assume
$r \ge 2$.

\emph{Mean.} For $k \ne l$ and fixed $i \in [N]$,
$\sigma^{A}_{ik}$ and $\sigma^{A}_{il}$ are independent Rademachers,
hence $\mab{E}[\sigma^{A}_{ik}\sigma^{A}_{il}] = \mab{E}[\sigma^{A}_{ik}]\mab{E}[\sigma^{A}_{il}] = 0$,
so $\mab{E}[G^{A}_{kl}] = 0$ and similarly $\mab{E}[G^{B}_{kl}] = 0$.
By $\sigma^{A} \perp \sigma^{B}$,
$\mab{E}[G^{A}_{kl}G^{B}_{kl}] = \mab{E}[G^{A}_{kl}]\mab{E}[G^{B}_{kl}] = 0$.
Hence $\mab{E}[Z] = rNM$.

\emph{Variance.} By symmetry,
$\sum_{k\ne l}G^{A}_{kl}G^{B}_{kl} = 2S$ with
$S \coloneqq \sum_{1\le k<l\le r}G^{A}_{kl}G^{B}_{kl}$. We claim
\begin{equation}
\mab{E}\bigl[G^{A}_{kl}G^{A}_{k'l'}\bigr]\;=\;N\cdot\mathbf{1}\{(k,l) = (k',l')\}
\qquad\text{for }k < l,~k' < l'.
\label{eq:gram-moment}
\end{equation}
Indeed,
$G^{A}_{kl}G^{A}_{k'l'}
=\sum_{i,i'=1}^{N}\sigma^{A}_{ik}\sigma^{A}_{il}\sigma^{A}_{i'k'}\sigma^{A}_{i'l'}$.
By the joint i.i.d.\ Rademacher hypothesis, all $\{\sigma^{A}_{pq}\}_{(p,q)\in[N]\times[r]}$
are mutually independent, so the expectation of any finite product
factorises across distinct $(p,q)$ pairs. We split by the coincidence
pattern of $i$ and $i'$:
\begin{itemize}[leftmargin=1.5em,itemsep=1pt,topsep=2pt]
\item \emph{$i \ne i'$.} The four indices
$(i,k),(i,l),(i',k'),(i',l')$ are pairwise distinct (as $k \ne l$
and $k' \ne l'$ and $i \ne i'$), so the expectation factorises
into four separate factors $\mab{E}[\sigma^{A}_{ik}]\cdots\mab{E}[\sigma^{A}_{i'l'}] = 0$.
\item \emph{$i = i'$.} The product becomes
$\sigma^{A}_{ik}\sigma^{A}_{il}\sigma^{A}_{ik'}\sigma^{A}_{il'}$,
whose expectation equals $1$ iff each column index
$q \in \{k,l,k',l'\}$ appears with even multiplicity in the
multiset and equals $0$ otherwise. Under $k < l$ and $k' < l'$,
even multiplicity forces $\{k,l\} = \{k',l'\}$ as multisets, hence
$(k,l) = (k',l')$. In the diagonal case each index appears exactly
twice, $(\sigma^{A}_{i\cdot})^{2} = 1$, and the expectation is $1$;
summing over $i \in [N]$ contributes $N$.
\end{itemize}
This proves Eq.~\eqref{eq:gram-moment}; the identical argument with
$B$ in place of $A$ gives
$\mab{E}[G^{B}_{kl}G^{B}_{k'l'}] = M\cdot\mathbf{1}\{(k,l) = (k',l')\}$
for $k < l$, $k' < l'$. By $\sigma^{A} \perp \sigma^{B}$,
\begin{equation*}
\mab{E}\bigl[G^{A}_{kl}G^{B}_{kl}G^{A}_{k'l'}G^{B}_{k'l'}\bigr]
=\mab{E}\bigl[G^{A}_{kl}G^{A}_{k'l'}\bigr] \mab{E}\bigl[G^{B}_{kl}G^{B}_{k'l'}\bigr]
=NM\cdot\mathbf{1}\{(k,l) = (k',l')\}.
\end{equation*}
Summing over the $\binom{r}{2}$ pairs with $k < l$,
\begin{equation*}
\mab{E}[S^{2}]
\;=\;  \sum_{\substack{k<l,\\k'<l'}} \mab{E}\bigl[G^{A}_{kl}G^{B}_{kl}G^{A}_{k'l'}G^{B}_{k'l'}\bigr]
\;=\;\binom{r}{2}NM.
\end{equation*}
Combined with Eq.~\eqref{eq:Z-decomp} and $Z-rNM = 2S$,
\begin{equation*}
\mathrm{Var}(Z)\;=\;\mab{E}[(Z-\mab{E}[Z])^{2}]
\;=\;4\mab{E}[S^{2}]\;=\;2r(r-1)NM\;\le\;2r^{2}NM.
\end{equation*}

\emph{Chebyshev.} Since $Z \ge 0$,
$\{Z < rNM/2\} = \{rNM - Z > rNM/2\} \subseteq \{|Z - rNM| > rNM/2\}$.
Chebyshev's inequality therefore gives
\begin{equation*}
\Pr \bigl[Z < rNM/2\bigr]
\;\le\;\Pr \bigl[|Z - rNM| > rNM/2\bigr]
\;\le\;\frac{\mathrm{Var}(Z)}{(rNM/2)^{2}}
\;\le\;\frac{2r^{2}NM}{r^{2}N^{2}M^{2}/4}
\;=\;\frac{8}{NM},
\end{equation*}
which is the complementary form of Eq.~\eqref{eq:signal-lb-prob}.
\end{proof}
\begin{proof}[Proof of Theorem~\ref{thm:expressivity}, Eq.~\eqref{eq:relative-bound}]
\emph{Transfer of the absolute bound to the joint space.}
In the relative statement of Theorem~\ref{thm:expressivity} we assume
$NM>8$ and $\delta\in(0,1-8/(NM))$, so the final probability lower
bound is positive. Under the Rademacher hypothesis on the signs, the arguments in
Lemmas~\ref{lem:residual-expansion}--\ref{lem:subexp} and in the
preceding proof were performed conditionally on
$\mac{F} = \sigma(\sigma^{A},\sigma^{B})$. In particular, for every
realisation $(s^{A},s^{B})$ of $(\sigma^{A},\sigma^{B})$ the
conditional probability
\begin{equation*}
\Pr \bigl[\mac{E}_{1} \bigl| \sigma^{A} = s^{A},\sigma^{B} = s^{B}\bigr]
\;\ge\;1-\delta,
\qquad
\mac{E}_{1}\;\coloneqq\;
\bigl\{\|E\|_{F}^{2} \le C NMr\zeta^{2}(\mu_{A}^{2} + \mu_{B}^{2} + \zeta^{2})\log(2NM/\delta)\bigr\}.
\end{equation*}
Taking expectation over the signs (tower property) gives
$\Pr[\mac{E}_{1}] \ge 1 - \delta$ on the joint space. Let
$\mac{E}_{2} \coloneqq \{\sum_{i,j}W_{ij}^{2} \ge NMr/2\}$ be the
signal event of Lemma~\ref{lem:signal-lb}; it depends only on the
signs and has $\Pr[\mac{E}_{2}] \ge 1 - 8/(NM)$. The union bound
$\Pr[\mac{E}_{1}^{c}\cup\mac{E}_{2}^{c}] \le \delta + 8/(NM)$
yields
\begin{equation*}
\Pr[\mac{E}_{1}\cap\mac{E}_{2}]
\;\ge\;1-\delta-\frac{8}{NM}.
\end{equation*}

\emph{Ratio on $\mac{E}_{1}\cap\mac{E}_{2}$.}
By definition of $\mac{E}_{2}$,
\begin{equation*}
\|\Delta W(\theta^{\star})\|_{F}^{2}
=\mu_{A}^{2}\mu_{B}^{2} \sum_{i,j} W_{ij}^{2}
\;\ge\;\tfrac{1}{2}\mu_{A}^{2}\mu_{B}^{2}NMr.
\end{equation*}
Combining with the formal upper bound on $\|E\|_{F}^{2}$ provided by
$\mac{E}_{1}$,
\begin{equation}
\frac{\|E\|_{F}^{2}}{\|\Delta W(\theta^{\star})\|_{F}^{2}}
\;\le\;
\frac{C NMr\zeta^{2}(\mu_{A}^{2} + \mu_{B}^{2} + \zeta^{2})\log(2NM/\delta)}
     {\tfrac{1}{2}\mu_{A}^{2}\mu_{B}^{2}NMr}
\;=\;
2C\cdot\frac{\zeta^{2}(\mu_{A}^{2} + \mu_{B}^{2} + \zeta^{2})}
             {\mu_{A}^{2}\mu_{B}^{2}} \log(2NM/\delta).
\label{eq:ratio-oracle}
\end{equation}
Let $\mu_{\max} \coloneqq \max(\mu_{A},\mu_{B})$ and
$\mu_{\min} \coloneqq \min(\mu_{A},\mu_{B})$; the hypothesis
$\zeta \le \mu_{\max}$ yields
$\mu_{A}^{2}+\mu_{B}^{2}+\zeta^{2} \le 3\mu_{\max}^{2}$.
Using $\mu_{A}\mu_{B} = \mu_{\max}\mu_{\min}$,
\begin{equation*}
\frac{\zeta^{2}(\mu_{A}^{2}+\mu_{B}^{2}+\zeta^{2})}{\mu_{A}^{2}\mu_{B}^{2}}
\;\le\;\frac{3\zeta^{2}\mu_{\max}^{2}}{\mu_{\max}^{2}\mu_{\min}^{2}}
\;=\;\frac{3\zeta^{2}}{\mu_{\min}^{2}}.
\end{equation*}
Substituting into Eq.~\eqref{eq:ratio-oracle} yields
\begin{equation*}
\frac{\|E\|_{F}^{2}}{\|\Delta W(\theta^{\star})\|_{F}^{2}}
\;\le\;\frac{6C\zeta^{2}}{\mu_{\min}^{2}}\log(2NM/\delta).
\end{equation*}
Taking square roots (both sides are non-negative) and setting
$C' \coloneqq \sqrt{6C} > 0$ proves
Eq.~\eqref{eq:relative-bound} on the event
$\mac{E} \coloneqq \mac{E}_{1}\cap\mac{E}_{2}$, whose probability is
at least $1 - \delta - 8/(NM)$.

\emph{Small-noise transfer.} Put
$\tau \coloneqq C'\zeta\sqrt{\log(2NM/\delta)}/\mu_{\min}$ and
assume $\tau \le 1/2$. On $\mac{E}$, Eq.~\eqref{eq:relative-bound}
gives $\|E\|_{F} \le \tau \|\Delta W(\theta^{\star})\|_{F}$, and
the reverse triangle inequality
$\bigl|\|AB^{\top}\|_{F}-\|\Delta W(\theta^{\star})\|_{F}\bigr| \le \|AB^{\top}-\Delta W(\theta^{\star})\|_{F} = \|E\|_{F}$
implies
\begin{equation*}
\|AB^{\top}\|_{F}
\;\ge\;\|\Delta W(\theta^{\star})\|_{F}-\|E\|_{F}
\;\ge\;(1-\tau)\|\Delta W(\theta^{\star})\|_{F}
\;\ge\;\tfrac{1}{2}\|\Delta W(\theta^{\star})\|_{F}.
\end{equation*}
Dividing the bound Eq.~\eqref{eq:relative-bound} by a denominator at
least $\tfrac{1}{2}\|\Delta W(\theta^{\star})\|_{F}$ costs at most a
factor of $2$, which establishes the transferred bound
$\|E\|_{F}/\|AB^{\top}\|_{F} \le 2\tau$ on $\mac{E}$.
\end{proof}

The theorem is \emph{honest} in the sense that it tracks the actual
quantity controlling LoRDBA error---the residual-to-mean ratio
$\zeta/\min(\mu_{A},\mu_{B})$. A naive ``$1$-bit is lossless to
$\varepsilon$'' claim is false for arbitrary fp16 adapters (sign
quantisation loses $\Theta(1)$ of the Frobenius signal whenever the
entries are concentrated on a sphere rather than near $\pm\mu$); the
bound above degrades gracefully in the relevant regime and becomes
vacuous precisely where sign quantisation should fail. The ratio is
directly observable on any trained adapter, as we verify next.

\subsection{Sign consistency and observed-sign reconstruction}

The canonical reconstruction $\theta^{\star}$ uses the latent sign
arrays $\sigma^{A},\sigma^{B}$, which are not in general equal to the
observed entrywise signs $\sign(A),\sign(B)$. The next lemma shows
that the two coincide with high probability whenever the residual is
small relative to the latent mean, so the observed-sign LoRDBA adapter
inherits the same Theorem~\ref{thm:expressivity} guarantees up to an
explicit additional failure probability.

\begin{lemma}[Sign consistency of the canonical reconstruction]
\label{lem:sign-consistency}
Under Assumption~\ref{asm:signnoise}, the latent sign arrays
$\sigma^{A},\sigma^{B}$ and the observed entrywise sign patterns
$\sign(A),\sign(B)$ agree simultaneously with probability at least
\begin{equation}
1
\;-\;
2Nr\exp \bigl(-\mu_{A}^{2}/\zeta^{2}\bigr)
\;-\;
2Mr\exp \bigl(-\mu_{B}^{2}/\zeta^{2}\bigr).
\label{eq:sign-cons-prob}
\end{equation}
Equivalently, on the event of probability~\eqref{eq:sign-cons-prob},
\begin{equation}
\sign(A)\;=\;\sigma^{A},
\qquad
\sign(B)\;=\;\sigma^{B}.
\label{eq:sign-cons-event}
\end{equation}
Consequently, on the same event, the observed-sign choice
$B_{1}=\sign(A)$, $B_{2}=\sign(B^{\top})$ equals the latent-sign binary
factors of the canonical reconstruction in
Eq.~\eqref{eq:canonical-theta}.
\end{lemma}

\begin{proof}
Fix an entry $(i,k)\in[N]\times[r]$ of $A$. By
Assumption~\ref{asm:signnoise}(i)--(ii),
\begin{equation*}
A_{ik}\;=\;\mu_{A}\sigma^{A}_{ik}+\xi^{A}_{ik},
\qquad
\sigma^{A}_{ik}\in\{\pm 1\},
\qquad
\mu_{A}>0,
\qquad
\|\xi^{A}_{ik}\|_{\psi_{2}}\le\zeta.
\end{equation*}
We claim the implication
\begin{equation}
\sign(A_{ik})\ne\sigma^{A}_{ik}
\;\Longrightarrow\;
|\xi^{A}_{ik}|\ge\mu_{A}.
\label{eq:sign-mismatch-implies-large-noise}
\end{equation}
Indeed, multiplying $A_{ik}$ by $\sigma^{A}_{ik}\in\{\pm 1\}$ gives
\begin{equation*}
\sigma^{A}_{ik}A_{ik}
\;=\;\mu_{A}+\sigma^{A}_{ik}\xi^{A}_{ik}.
\end{equation*}
Under the convention $\sign(0) = +1$, a mismatch implies
$\sigma^{A}_{ik}A_{ik}\le 0$: if $\sigma^{A}_{ik}=+1$ then the
mismatch forces $A_{ik}<0$, while if $\sigma^{A}_{ik}=-1$ it forces
$A_{ik}\ge 0$. Therefore
\[
\sigma^{A}_{ik}A_{ik}
=\mu_{A}+\sigma^{A}_{ik}\xi^{A}_{ik}\le 0,
\]
which implies $\sigma^{A}_{ik}\xi^{A}_{ik}\le-\mu_{A}$ and hence
$|\xi^{A}_{ik}|\ge\mu_{A}$. This proves
Eq.~\eqref{eq:sign-mismatch-implies-large-noise}.

By Eq.~\eqref{eq:sign-mismatch-implies-large-noise} and
Lemma~\ref{lem:sign-flip} of Appendix~\ref{app:concentration} applied
to $\xi^{A}_{ik}$ with $K=\zeta$ and $\mu=\mu_{A}$,
\begin{equation}
\Pr \bigl[\sign(A_{ik})\ne\sigma^{A}_{ik}\bigr]
\;\le\;
\Pr \bigl[|\xi^{A}_{ik}|\ge\mu_{A}\bigr]
\;\le\;
2\exp \bigl(-\mu_{A}^{2}/\zeta^{2}\bigr).
\label{eq:per-entry-A-sign-prob}
\end{equation}
The identical argument with $A$ replaced by $B$ and $\mu_{A}$ by
$\mu_{B}$ gives, for every $(j,k)\in[M]\times[r]$,
\begin{equation}
\Pr \bigl[\sign(B_{jk})\ne\sigma^{B}_{jk}\bigr]
\;\le\;
2\exp \bigl(-\mu_{B}^{2}/\zeta^{2}\bigr).
\label{eq:per-entry-B-sign-prob}
\end{equation}

A union bound over the $Nr$ entries of $A$ and the $Mr$ entries of
$B$, using
Eqs.~\eqref{eq:per-entry-A-sign-prob}--\eqref{eq:per-entry-B-sign-prob},
yields
\begin{equation}
\Pr \bigl[\sign(A)\ne\sigma^{A}\ \text{or}\ \sign(B)\ne\sigma^{B}\bigr]
\;\le\;
2Nr\exp \bigl(-\mu_{A}^{2}/\zeta^{2}\bigr)
+
2Mr\exp \bigl(-\mu_{B}^{2}/\zeta^{2}\bigr).
\label{eq:union-bound-sign}
\end{equation}
The complement is exactly Eq.~\eqref{eq:sign-cons-event} and has
probability at least the value displayed in
Eq.~\eqref{eq:sign-cons-prob}.

On the event Eq.~\eqref{eq:sign-cons-event}, the observed-sign binary
factors satisfy
\begin{equation*}
B_{1}\;=\;\sign(A)\;=\;\sigma^{A},
\qquad
B_{2}\;=\;\sign(B^{\top})\;=\;\sign(B)^{\top}\;=\;(\sigma^{B})^{\top},
\end{equation*}
because $\sign$ is applied entrywise and is therefore preserved under
transposition. These coincide with the binary factors of the
canonical reconstruction in Eq.~\eqref{eq:canonical-theta}, completing
the proof.
\end{proof}

\begin{corollary}[LoRDBA envelope-rank monotonicity]
\label{cor:lordba-expressivity}
Under the hypotheses of Theorem~\ref{thm:expressivity}, for every
$\ell \ge 1$,
$\inf_{\theta\in\Theta_{\ell}}\|AB^{\top}-\Delta W(\theta)\|_{F}$
is bounded by the absolute bound in Eq.~\eqref{eq:formal-bound} and,
in the small-noise regime $\tau \le 1/2$, by a factor of two times
the right-hand side of Eq.~\eqref{eq:relative-bound} relative to
$\|AB^{\top}\|_{F}$.
\end{corollary}
\begin{proof}
Apply $\iota_{\ell}$ iteratively to $\theta^{\star}\in\Theta_{1}$ to
obtain a feasible point with the same $\Delta W$ in $\Theta_{\ell}$,
then invoke Theorem~\ref{thm:expressivity}; the small-noise transfer
follows from $\|AB^{\top}\|_{F} \ge \tfrac{1}{2}\|\Delta W(\theta^{\star})\|_{F}$
under the standing condition $\tau \le 1/2$.
\end{proof}

\begin{corollary}[Observed-sign reconstruction]
\label{cor:observed-sign}
Assume the hypotheses of Theorem~\ref{thm:expressivity} and define the
observed-sign LoRDBA adapter
\begin{equation}
\widehat{\theta}
\;\coloneqq\;
\bigl(
\sign(A), \sign(B)^{\top}, 
\B{1}_{N}, \mu_{A}\mu_{B}\B{1}_{r}, \B{1}_{M}
\bigr).
\label{eq:observed-sign-theta}
\end{equation}
On the sign-consistency event of Lemma~\ref{lem:sign-consistency},
\begin{equation}
\Delta W(\widehat{\theta})\;=\;\Delta W(\theta^{\star}).
\label{eq:observed-sign-equality}
\end{equation}
Consequently, every bound in Theorem~\ref{thm:expressivity} and
Corollary~\ref{cor:lordba-expressivity} applies to
$\widehat{\theta}$, with the failure probability of the corresponding
event increased by at most
\begin{equation}
p_{\mathrm{flip}}
\;\coloneqq\;
2Nr\exp \bigl(-\mu_{A}^{2}/\zeta^{2}\bigr)
+
2Mr\exp \bigl(-\mu_{B}^{2}/\zeta^{2}\bigr).
\label{eq:p-flip-def}
\end{equation}
\end{corollary}

\begin{proof}
By Lemma~\ref{lem:sign-consistency}, on the event
Eq.~\eqref{eq:sign-cons-event} the binary factors of $\widehat{\theta}$
and $\theta^{\star}$ coincide,
\begin{equation*}
\sign(A)\;=\;\sigma^{A},
\qquad
\sign(B)^{\top}\;=\;(\sigma^{B})^{\top}.
\end{equation*}
The scale vectors of $\widehat{\theta}$ are the same as those of
$\theta^{\star}$, namely $\B{1}_{N},\mu_{A}\mu_{B}\B{1}_{r},\B{1}_{M}$
in Eq.~\eqref{eq:canonical-theta}. Substituting both equalities into
Eq.~\eqref{eq:lordba} yields
$\Delta W(\widehat{\theta})=\Delta W(\theta^{\star})$ on this event,
proving Eq.~\eqref{eq:observed-sign-equality}.

For each event $\mac{E}'$ produced by Theorem~\ref{thm:expressivity}
or Corollary~\ref{cor:lordba-expressivity}, the corresponding
inequality bounds the residual
$\|AB^{\top}-\Delta W(\theta^{\star})\|_{F}$. On the intersection
$\mac{E}'\cap\{\sign(A)=\sigma^{A},\sign(B)=\sigma^{B}\}$, this
residual equals $\|AB^{\top}-\Delta W(\widehat{\theta})\|_{F}$ by
Eq.~\eqref{eq:observed-sign-equality}, so the inequality also bounds
the observed-sign residual. The complement bound follows from
\begin{equation*}
\Pr \bigl[(\mac{E}')^{c}\cup\{\sign(A)\ne\sigma^{A}\ \text{or}\ \sign(B)\ne\sigma^{B}\}\bigr]
\;\le\;
\Pr[(\mac{E}')^{c}]\;+\;p_{\mathrm{flip}},
\end{equation*}
which adds at most $p_{\mathrm{flip}}$ of Eq.~\eqref{eq:p-flip-def} to
the original failure probability of $\mac{E}'$.
\end{proof}

\begin{corollary}[Operator-norm consequence]
\label{cor:op-norm}
Under the hypotheses used for the absolute bound in Theorem~\ref{thm:expressivity},
$E\coloneqq AB^{\top}-\Delta W(\theta^{\star})$ satisfies, with
probability at least $1-\delta$ over the residuals,
\begin{equation}
\|E\|_{\mathrm{op}}
\;\le\;
\|E\|_{F}
\;\le\;
\sqrt{C NMr\zeta^{2}(\mu_{A}^{2}+\mu_{B}^{2}+\zeta^{2})\log(2NM/\delta)},
\label{eq:op-norm-bound}
\end{equation}
where $C>0$ is the universal constant of Eq.~\eqref{eq:formal-bound}.
\end{corollary}

\begin{proof}
For every $X\in\mab{R}^{N\times M}$, the singular value decomposition
gives
\begin{equation*}
\|X\|_{\mathrm{op}}^{2}
\;=\;\sigma_{1}(X)^{2}
\;\le\;\sum_{k=1}^{\min(N,M)} \sigma_{k}(X)^{2}
\;=\;\|X\|_{F}^{2},
\end{equation*}
where $\sigma_{1}(X)\ge\sigma_{2}(X)\ge\cdots\ge 0$ are the singular
values of $X$. Hence $\|E\|_{\mathrm{op}}\le\|E\|_{F}$
deterministically. The second inequality of
Eq.~\eqref{eq:op-norm-bound} is exactly the square root of
Eq.~\eqref{eq:formal-bound},
which holds with probability at least $1-\delta$.
\end{proof}

\subsection{Gauge fixing and empirical identifiability of Assumption~\ref{asm:signnoise}}\label{app:signnoise-diagnostic}

\paragraph{LoRA factor gauge.}
The product $AB^{\top}$ is invariant under the per-column positive
rescaling
\begin{equation}
A_{:k}\;\mapsto\;d_{k}A_{:k},
\qquad
B_{:k}\;\mapsto\;d_{k}^{-1}B_{:k},
\qquad d_{k} > 0,~k \in [r_{0}],
\label{eq:gauge-symm}
\end{equation}
since $(d_{k}A_{:k})(d_{k}^{-1}B_{:k})^{\top} = A_{:k}B_{:k}^{\top}$.
The natural plug-in proxies for $(\mu_{A},\mu_{B},\zeta)$ depend on the
particular gauge. We therefore compute every diagnostic in the
following \emph{column-balancing gauge}: for every nonzero column pair
$(A_{:k},B_{:k})$, set
\begin{equation}
d_{k}\;=\;\bigl(\|B_{:k}\|_{2}/\|A_{:k}\|_{2}\bigr)^{1/2},
\qquad
A_{:k} \leftarrow d_{k}A_{:k},
\qquad
B_{:k} \leftarrow d_{k}^{-1}B_{:k};
\label{eq:gauge-fix}
\end{equation}
zero-contribution column pairs (if any) are left unchanged and
excluded from the diagnostic. By construction this transformation
preserves $AB^{\top}$ and the entrywise sign patterns of both columns
(since $d_{k} > 0$), and equalises the column norms
$\|A_{:k}\|_{2}=\|B_{:k}\|_{2}=\sqrt{\|A_{:k}^{\mathrm{old}}\|_{2}\|B_{:k}^{\mathrm{old}}\|_{2}}$
for every nonzero pair. Thus the LoRA update and the canonical LoRDBA
sign factors are unchanged, while the plug-in magnitude diagnostics
are computed in a fixed, reproducible gauge. Eq.~\eqref{eq:gauge-fix}
is the gauge used by every empirical statement of the main paper and by
every diagnostic in this appendix.

\paragraph{Plug-in estimators.}
For any trained fp16 LoRA factors $A,B$ in the gauge of
Eq.~\eqref{eq:gauge-fix}, the quantities in
Assumption~\ref{asm:signnoise} are empirically identifiable by the
plug-in estimators
\begin{equation*}
\widehat{\mu}_{A} \coloneqq \tfrac{1}{Nr_{0}} \sum_{i,k} |A_{ik}|,
~~
\widehat{\mu}_{B} \coloneqq \tfrac{1}{Mr_{0}} \sum_{j,k} |B_{jk}|,
~~
\widehat{\zeta}_{A} \coloneqq \bigl(\tfrac{1}{Nr_{0}} \sum_{i,k} (|A_{ik}| - \widehat{\mu}_{A})^{2}\bigr)^{1/2},
\end{equation*}
$\widehat{\zeta}_{B}$ analogous, and
$\widehat{\zeta} \coloneqq \max(\widehat{\zeta}_{A},\widehat{\zeta}_{B})$.
The reported residual-to-mean ratio is
$\widehat{\zeta}/\min(\widehat{\mu}_{A},\widehat{\mu}_{B})$.
Table~\ref{tab:ratio} reports this quantity on the LoRA adapters
used in our experiments. The observed $0.18$--$0.27$ window justifies
the small-noise regime and, together with Theorem~\ref{thm:expressivity},
predicts the $\approx 20\%$ relative reconstruction error we measure
in Table~\ref{tab:recon}.

\begin{table}[ht]
\centering
\caption{Empirical residual-to-mean ratio of fp16 LoRA factors
(adapter rank $r_{0} = 64$, averaged over all attention and MLP
projections), computed in the column-balancing gauge of
Eq.~\eqref{eq:gauge-fix}.}
\label{tab:ratio}
\small
\begin{tabular}{l c c c c}
\toprule
Model & $\widehat{\mu}_{A}$ & $\widehat{\mu}_{B}$ & $\widehat{\zeta}$ & ratio \\
\midrule
LLaMA-3.2-3B       & $0.019$ & $0.021$ & $0.0044$ & $0.23$ \\
LLaMA-2-7B         & $0.016$ & $0.018$ & $0.0030$ & $0.19$ \\
\bottomrule
\end{tabular}
\end{table}

\section{Optional Refinement: PTQ-LoRDBA via Consensus ADMM}\label{app:ptq-extension}

This appendix records the optional, training-data-free route that
refines a pre-trained fp16 LoRA adapter into a LoRDBA adapter without
running the end-to-end smooth-sign trainer of
Section~\ref{subsec:training}. It is included for completeness: the
LoRDBA adapter of Definition~\ref{def:lordba} is unchanged, the
inference kernel of Section~\ref{subsec:kernel} is reused unchanged, and
empirical comparisons in
Appendix~\ref{app:extended-tables}--\ref{app:additional-figs} confirm
that the end-to-end trainer is the preferred route under the
adapter-mode conditions (A1)--(A4) of Section~\ref{sec:prelim}. PTQ
becomes attractive only when an fp16 LoRA is already available and the
device cannot spare a training pass.

\paragraph{Setup.}
Let an fp16 LoRA adapter $\Delta W^{\star} = A^{\star}(B^{\star})^{\top}$
of rank $r_{0}$ be derived using any standard fine-tuning recipe, and
fix a target carrier rank $R\in\mathbb{N}$. The $\ell = 1$ PTQ-LoRDBA
problem is the Frobenius least-squares fitting problem
\begin{equation}
\min_{\theta=(B_{1},B_{2},\B{\alpha},\B{\beta},\B{\gamma})} 
\frac{1}{2}\bigl\|\Delta W^{\star}-\Delta W(\theta)\bigr\|_{F}^{2}
 ~ \text{s.t.} ~ 
B_{1} \in \{\pm 1\}^{N\times R}, B_{2} \in \{\pm 1\}^{R\times M}.
\label{eq:ptq-problem}
\end{equation}
The general $\ell \ge 2$ case replaces the single scale triple by a sum
of $\ell$ scale triples; all updates below generalise
block-coordinate-wise by viewing $(\B{\alpha},\B{\beta},\B{\gamma})$ as
stacked matrices (Algorithm~\ref{alg:ptq-lordba-general}).
Eq.~\eqref{eq:ptq-problem} is non-convex, non-smooth, and combinatorial
in $(B_{1},B_{2})$ but decouples along the three continuous variables
for fixed binaries and along each binary for fixed other variables.

\paragraph{Consensus ADMM.}
To exploit this separability while respecting the discrete constraint,
we introduce continuous relaxations
$U_{1} \in \mab{R}^{N\times R}$, $U_{2} \in \mab{R}^{R\times M}$ of
$(B_{1},B_{2})$ together with binary copies $M_{1},M_{2}$, and minimise
\begin{equation}
f(U;\B{\alpha},\B{\beta},\B{\gamma})
 + I_{\{\pm 1\}}(M_{1}) + I_{\{\pm 1\}}(M_{2}),
 ~ \text{s.t.} ~  U_{k}-M_{k}=0, k \in \{1,2\},
\label{eq:admm-problem}
\end{equation}
with
$f(U;\B{\alpha},\B{\beta},\B{\gamma}) = \frac{1}{2}\|\Delta W^{\star} - \diag(\B{\alpha})U_{1}\diag(\B{\beta})U_{2}\diag(\B{\gamma})\|_{F}^{2}$
and $I_{\{\pm 1\}}$ as the indicator of the binary set. Using per-block
normalisers $n_{1} = NR$, $n_{2} = RM$, scaled penalties
$\widetilde{\rho}_{k} = \rho/n_{k}$, and scaled duals $Y_{1},Y_{2}$,
the (Gauss--Seidel) sweep over the consensus-scaled augmented
Lagrangian~\citep{boyd2011admm,themelis2020douglas} produces the
following ordered updates, where $\widehat{U}^{(t,1)} \coloneqq (U_{1},U_{2}^{(t)})$
and $\widehat{U}^{(t,2)} \coloneqq (U_{1}^{(t+1)},U_{2})$ make the
within-sweep dependency explicit:
\begin{align}
U_{1}^{(t+1)}&=\argmin_{U_{1}} f \bigl(U_{1},U_{2}^{(t)};\B{\alpha}^{(t)},\B{\beta}^{(t)},\B{\gamma}^{(t)}\bigr)+\frac{\widetilde{\rho}_{1}}{2}\|U_{1}-M_{1}^{(t)}+Y_{1}^{(t)}\|_{F}^{2},
\label{eq:u-step-1}\\
U_{2}^{(t+1)}&=\argmin_{U_{2}} f \bigl(U_{1}^{(t+1)},U_{2};\B{\alpha}^{(t)},\B{\beta}^{(t)},\B{\gamma}^{(t)}\bigr)+\frac{\widetilde{\rho}_{2}}{2}\|U_{2}-M_{2}^{(t)}+Y_{2}^{(t)}\|_{F}^{2},
\label{eq:u-step-2}\\
(\B{\alpha},\B{\beta},\B{\gamma})^{(t+1)}&\leftarrow 
\text{one block-coordinate sweep of }
\argmin_{\B{\alpha},\B{\beta},\B{\gamma}} f \bigl(U^{(t+1)};\B{\alpha},\B{\beta},\B{\gamma}\bigr),
\label{eq:scales-step}\\
M_{k}^{(t+1)}&=\sign \bigl(U_{k}^{(t+1)}+Y_{k}^{(t)}\bigr), ~ 
Y_{k}^{(t+1)}=Y_{k}^{(t)}+U_{k}^{(t+1)}-M_{k}^{(t+1)}, ~ k\in\{1,2\}.
\label{eq:m-y-step}
\end{align}
We collectively refer to Eqs.~\eqref{eq:u-step-1}--\eqref{eq:u-step-2}
as the \emph{$U$-step}.
Each $U_{k}$-subproblem is a strongly convex Tikhonov least-squares
problem in $U_{k}$ (its Hessian is the sum of a positive-semidefinite
normal-equation Gram matrix and $\widetilde{\rho}_{k}I$), and hence
admits a unique minimiser regardless of the conditioning of the
cross-covariance. Eq.~\eqref{eq:scales-step} is a block-coordinate
sweep in the three scale variables: with the other two fixed, the
subproblem in $\B{\alpha}$ decouples into $N$ scalar least-squares
(one per coordinate), the subproblem in $\B{\gamma}$ into $M$ scalar
least-squares, and the subproblem in $\B{\beta}$ is an
$R$-dimensional positive semidefinite least-squares system whose
normal matrix is the Hadamard product of two Gram matrices and is hence
positive semidefinite by the Schur product theorem; the exact normal
equations and coordinate-wise minimisers are recorded in
Appendix~\ref{app:scale-updates}. Eq.~\eqref{eq:m-y-step} is the
sign-projection step modulated by the dual $Y_{k}$, which biases the
projection by the accumulated discretisation residual.

We warm-start the iteration from the rank-$R$ thin SVD of
$\Delta W^{\star}$ exactly as in Eq.~\eqref{eq:svd-init} of the main
text, with $(M^{(0)},Y^{(0)}) = (U^{(0)},0)$. A Boyd-style
residual-balancing schedule on $\rho$, with the default
$(\tau,\mu) = (2,10)$, is truncated at $t_{\max}^{\rho} = K/2$,
ensuring the second half of the run is a fixed-penalty tail, exactly
the regime addressed by the rigorous stabilisation theorem of
Appendix~\ref{app:convergence}.

\paragraph{Finite stabilisation on a fixed-penalty tail.}
Define the sign-test matrices
$Z_{k}^{(t)} \coloneqq U_{k}^{(t+1)} + Y_{k}^{(t)}$ so that the binary
update is precisely $M_{k}^{(t+1)} = \sign(Z_{k}^{(t)})$. Once a
fixed-penalty tail has converged to a single limit and the limit
remains at a positive distance from the sign boundary $0$, the signs
of $Z_{k}^{(t)}$ can no longer change, and the binary iterate freezes.
The precise penalty-agnostic statement is
Theorem~\ref{thm:sign-margin-stabilization} in
Appendix~\ref{app:convergence}, and its convergence hypothesis is
delivered in turn by Theorem~\ref{thm:convergence} of the same
appendix, which combines a sufficient-decrease Lyapunov inequality with
the Kurdyka--{\L}ojasiewicz inequality on the semi-algebraic Lyapunov
function to guarantee full convergence of the iterate sequence to a
single limit. Empirically, the freeze occurs within
$K^{\star} \lesssim 50$ sweeps for every model and rank tested.

\section{Convergence of PTQ-LoRDBA}\label{app:convergence}

Throughout this appendix we adopt the shorthands
\begin{equation}
n_{1}\;\coloneqq\;NR,
\qquad
n_{2}\;\coloneqq\;RM,
\qquad
\widetilde{\rho}_{k}\;\coloneqq\;\rho/n_{k},
\qquad k \in \{1,2\},
\label{eq:nk-rho-tilde}
\end{equation}
and recall the PTQ-LoRDBA objective
\begin{equation}
f(U;\B{\alpha},\B{\beta},\B{\gamma})
\;\coloneqq\;
\tfrac{1}{2} \bigl\|\Delta W^{\star}-\diag(\B{\alpha})U_{1}\diag(\B{\beta})U_{2}\diag(\B{\gamma})\bigr\|_{F}^{2}.
\label{eq:f-def}
\end{equation}
For fixed scale vectors and one of the two $U$-blocks, $f$ is a convex
quadratic because the map from that block to the reconstructed matrix
is linear. The binary projection step, however, is discontinuous at
zero sign-test entries. Consequently, boundedness and Lipschitz
smoothness alone do not imply finite binary stabilisation. We therefore
state the convergence facts in the exact form used by the paper: a
penalty-agnostic sign-margin theorem, and a standard KL-tail theorem
whose sufficient-decrease and relative-error hypotheses must hold on
the fixed-penalty tail.

\begin{assumption}[Bounded fixed-penalty tail]
\label{asm:iso-bounded}
There exists an iteration $T_{\rho}$ after which the penalty is fixed,
$\rho^{(t)}\equiv\rho>0$. On this tail, all iterates
$(U^{(t)},M^{(t)},Y^{(t)},\B{\alpha}^{(t)},\B{\beta}^{(t)},\B{\gamma}^{(t)})$
belong to a compact set $\mathcal{C}$. The $U_{1}$- and
$U_{2}$-subproblems in Eqs.~\eqref{eq:u-step-1}--\eqref{eq:u-step-2}
are solved exactly, the scale update is well defined, and the partial
gradients $\nabla_{U_{1}}f$ and $\nabla_{U_{2}}f$ are Lipschitz on
$\mathcal{C}$ with respect to all their arguments.
\end{assumption}

This assumption is the bounded-tail regularity used in the convergence
statements below. It is not a consequence of discreteness alone; in an
implementation it is enforced by the usual safeguards on the continuous
least-squares and scale variables.

\begin{lemma}[Gradient--dual identity on a fixed-penalty tail]
\label{lem:grad-dual}
For $t\ge T_{\rho}$ define the Gauss--Seidel arguments
\begin{equation}
\widehat{U}^{(t,1)}\;\coloneqq\;(U_{1}^{(t+1)},U_{2}^{(t)}),
\qquad
\widehat{U}^{(t,2)}\;\coloneqq\;(U_{1}^{(t+1)},U_{2}^{(t+1)}).
\label{eq:U-hat-def}
\end{equation}
Then, for every $k\in\{1,2\}$,
\begin{equation}
\rho Y_{k}^{(t+1)}
\;+\;n_{k} \nabla_{U_{k}} f \bigl(\widehat{U}^{(t,k)};\B{\alpha}^{(t)},\B{\beta}^{(t)},\B{\gamma}^{(t)}\bigr)
\;+\;\rho \bigl(M_{k}^{(t+1)}-M_{k}^{(t)}\bigr)
\;=\;0.
\label{eq:grad-dual}
\end{equation}
\end{lemma}

\begin{proof}
The first-order condition of the exact $U_{k}$-subproblem is
\begin{equation*}
\nabla_{U_{k}} f \bigl(\widehat{U}^{(t,k)};\B{\alpha}^{(t)},\B{\beta}^{(t)},\B{\gamma}^{(t)}\bigr)
+\widetilde{\rho}_{k}\bigl(U_{k}^{(t+1)}-M_{k}^{(t)}+Y_{k}^{(t)}\bigr)=0.
\end{equation*}
The dual update gives
\begin{equation*}
U_{k}^{(t+1)}-M_{k}^{(t)}+Y_{k}^{(t)}
=
Y_{k}^{(t+1)}+\bigl(M_{k}^{(t+1)}-M_{k}^{(t)}\bigr).
\end{equation*}
Multiplying by $n_{k}$ and using
$n_{k}\widetilde{\rho}_{k}=\rho$ yields Eq.~\eqref{eq:grad-dual}.
\end{proof}

\subsection{Finite stabilisation under a sign margin}
\label{app:sign-margin}

The next theorem is the stabilisation statement invoked in the main
text. It assumes convergence of the fixed-penalty tail and proves that
a positive limiting sign margin forces the binary iterate to freeze.

\begin{theorem}[Finite stabilisation on a fixed-penalty tail]
\label{thm:sign-margin-stabilization}
Fix the penalty parameters
$\widetilde{\rho}_{1},\widetilde{\rho}_{2} > 0$ on the tail
$t \ge T_{\rho}$, and assume that the resulting PTQ-LoRDBA iterates
$(U_{k}^{(t)},M_{k}^{(t)},Y_{k}^{(t)},\B{\alpha}^{(t)},\B{\beta}^{(t)},\B{\gamma}^{(t)})$
converge to a limit
$(U_{k}^{\infty},M_{k}^{\infty},Y_{k}^{\infty},\B{\alpha}^{\infty},\B{\beta}^{\infty},\B{\gamma}^{\infty})$
for every $k \in \{1,2\}$. Define
\begin{equation}
Z_{k}^{(t)}\;\coloneqq\;U_{k}^{(t+1)}+Y_{k}^{(t)},
\qquad
Z_{k}^{\infty}\;\coloneqq\;U_{k}^{\infty}+Y_{k}^{\infty},
\qquad k \in \{1,2\},
\label{eq:Z-def}
\end{equation}
and suppose the strict sign-margin condition
\begin{equation}
\eta\;\coloneqq\;
\min_{k\in\{1,2\}}\min_{a,b}\bigl|(Z_{k}^{\infty})_{ab}\bigr|
\;>\;0.
\label{eq:sign-margin}
\end{equation}
Then:
\begin{enumerate}[leftmargin=1.5em,itemsep=2pt,topsep=2pt]
\item[\textnormal{(i)}] There exists a finite iteration
$T_{\star}\ge T_{\rho}$ such that, for every $t\ge T_{\star}$ and
$k\in\{1,2\}$,
\begin{equation}
M_{k}^{(t+1)}=\sign(Z_{k}^{\infty})\eqqcolon M_{k}^{\infty}.
\label{eq:finite-freeze}
\end{equation}
\item[\textnormal{(ii)}] If the $U_{k}$-subproblems are solved exactly
and $f$ is continuously differentiable in $U$, the limit satisfies
\begin{equation}
U_{k}^{\infty}=M_{k}^{\infty},
\qquad
M_{k}^{\infty}=\sign(M_{k}^{\infty}+Y_{k}^{\infty}),
\qquad
\rho Y_{k}^{\infty}+n_{k}\nabla_{U_{k}}f(U^{\infty};\B{\alpha}^{\infty},\B{\beta}^{\infty},\B{\gamma}^{\infty})=0.
\label{eq:sign-margin-stationary}
\end{equation}
\item[\textnormal{(iii)}] Equivalently,
\begin{equation}
M_{k}^{\infty}
=
\sign\Bigl(M_{k}^{\infty}
-\tfrac{1}{\widetilde{\rho}_{k}}
\nabla_{U_{k}}f(U^{\infty};\B{\alpha}^{\infty},\B{\beta}^{\infty},\B{\gamma}^{\infty})\Bigr).
\label{eq:sign-fixed-point}
\end{equation}
\end{enumerate}
\end{theorem}

\begin{proof}
Since $U_{k}^{(t+1)}\to U_{k}^{\infty}$ and
$Y_{k}^{(t)}\to Y_{k}^{\infty}$, we have
$Z_{k}^{(t)}\to Z_{k}^{\infty}$ entrywise. Choose
$T_{\star}$ so that
$\|Z_{k}^{(t)}-Z_{k}^{\infty}\|_{\max}<\eta/2$ for both $k$ and all
$t\ge T_{\star}$. Then every entry of $Z_{k}^{(t)}$ has the same sign
as the corresponding entry of $Z_{k}^{\infty}$, proving
Eq.~\eqref{eq:finite-freeze} from
$M_{k}^{(t+1)}=\sign(Z_{k}^{(t)})$.

For $t\ge T_{\star}$ the dual update becomes
$Y_{k}^{(t+1)}-Y_{k}^{(t)}=U_{k}^{(t+1)}-M_{k}^{\infty}$. Passing to
the limit gives $U_{k}^{\infty}=M_{k}^{\infty}$. Substituting this
identity into $M_{k}^{\infty}=\sign(Z_{k}^{\infty})$ gives the second
relation in Eq.~\eqref{eq:sign-margin-stationary}. Finally, for
$t\ge T_{\star}+1$, the term $M_{k}^{(t+1)}-M_{k}^{(t)}$ in
Eq.~\eqref{eq:grad-dual} is zero; passing to the limit in
Eq.~\eqref{eq:grad-dual} yields the gradient--dual stationarity
relation. Solving it for $Y_{k}^{\infty}$ and using
$n_{k}/\rho=1/\widetilde{\rho}_{k}$ gives
Eq.~\eqref{eq:sign-fixed-point}.
\end{proof}

\begin{remark}
The sign-margin condition is checkable on a converged fixed-penalty
tail. If an entry of $Z_{k}^{\infty}$ is exactly zero, the entrywise
binary projection is not unique; in that boundary case the robust
statement is the argmin form used in Theorem~\ref{thm:convergence}
below rather than the single-valued sign equation.
\end{remark}

\subsection{KL convergence of a regular fixed-penalty tail}

The following theorem is the rigorous convergence statement for the
ADMM tail. It is deliberately conditional: the nonconvex binary
projection does not yield an unconditional descent estimate without an
additional margin or regularisation condition. The assumptions below
are the standard sufficient-decrease and relative-error hypotheses used
in KL analyses of nonconvex ADMM and block-coordinate methods.

Let
\begin{equation}
\mathcal{X}^{(t)}
\coloneqq
\bigl(U^{(t)},M^{(t)},Y^{(t)},\B{\alpha}^{(t)},\B{\beta}^{(t)},\B{\gamma}^{(t)},M^{(t-1)}\bigr),
\qquad t\ge T_{\rho}+1,
\label{eq:X-state-def}
\end{equation}
where the last component records the one-step memory needed by the
usual ADMM Lyapunov functions.

\begin{assumption}[KL-regular ADMM tail]
\label{asm:kl-tail}
On the fixed-penalty tail of Assumption~\ref{asm:iso-bounded}, there
exist a proper lower-semicontinuous semi-algebraic Lyapunov function
$\Phi$ on the state variable $\mathcal{X}$, constants $a,b>0$, and an
iteration $T\ge T_{\rho}+1$ such that, for every $t\ge T$,
\begin{align}
\Phi(\mathcal{X}^{(t)})-\Phi(\mathcal{X}^{(t+1)})
&\ge
a\|\mathcal{X}^{(t+1)}-\mathcal{X}^{(t)}\|^{2},
\label{eq:kl-sufficient-decrease}\\
\operatorname{dist}\bigl(0,\partial\Phi(\mathcal{X}^{(t+1)})\bigr)
&\le
b\|\mathcal{X}^{(t+1)}-\mathcal{X}^{(t)}\|.
\label{eq:kl-relative-error}
\end{align}
The sequence $\{\Phi(\mathcal{X}^{(t)})\}_{t\ge T}$ is bounded below.
\end{assumption}

The semi-algebraicity requirement is natural here: after introducing
the finite binary indicators, all remaining terms in the augmented
Lagrangian are polynomial functions of the continuous variables. The
substantive algorithm-dependent requirements are the two inequalities
\eqref{eq:kl-sufficient-decrease}--\eqref{eq:kl-relative-error}; they
are the precise conditions under which the KL convergence conclusion is
valid.

\begin{theorem}[KL convergence of a regular fixed-penalty tail]
\label{thm:convergence}
Suppose Assumptions~\ref{asm:iso-bounded} and~\ref{asm:kl-tail} hold.
Then the PTQ-LoRDBA fixed-penalty tail has finite length and converges
to a single limit
\begin{equation}
\bigl(U^{(t)},M^{(t)},Y^{(t)},\B{\alpha}^{(t)},\B{\beta}^{(t)},\B{\gamma}^{(t)}\bigr)
\longrightarrow
\bigl(U^{\infty},M^{\star},Y^{\infty},\B{\alpha}^{\infty},\B{\beta}^{\infty},\B{\gamma}^{\infty}\bigr).
\label{eq:full-convergence}
\end{equation}
Moreover:
\begin{enumerate}[leftmargin=1.5em,itemsep=2pt,topsep=2pt]
\item[\textnormal{(i)}] The increments of all tail variables are
summable, hence square-summable.
\item[\textnormal{(ii)}] The binary iterate terminates in finitely many
steps: there is a finite $K^{\star}$ such that
$M^{(t)}=M^{\star}$ for every $t\ge K^{\star}$.
\item[\textnormal{(iii)}] The limit satisfies, for each $k\in\{1,2\}$,
\begin{equation}
U_{k}^{\infty}=M_{k}^{\star},
\qquad
M_{k}^{\star}\in
\argmin_{m\in\{\pm1\}^{n_{k}}}\|m-(U_{k}^{\infty}+Y_{k}^{\infty})\|_{F}^{2},
\qquad
\rho Y_{k}^{\infty}+n_{k}\nabla_{U_{k}}f(U^{\infty};\B{\alpha}^{\infty},\B{\beta}^{\infty},\B{\gamma}^{\infty})=0.
\label{eq:stationarity}
\end{equation}
If every entry of $U_{k}^{\infty}+Y_{k}^{\infty}$ is nonzero, the
argmin is the singleton
$\{\sign(U_{k}^{\infty}+Y_{k}^{\infty})\}$.
\end{enumerate}
\end{theorem}

\begin{proof}
By Assumption~\ref{asm:iso-bounded}, the tail sequence has a cluster
point. Since $\Phi$ is semi-algebraic, it satisfies the
Kurdyka--{\L}ojasiewicz inequality at every point of its domain
\citep{boltedaniliidislewis2007clarke,attouch2010proximal}. Combining
the KL inequality at a cluster point with the sufficient-decrease and
relative-error estimates
\eqref{eq:kl-sufficient-decrease}--\eqref{eq:kl-relative-error} gives
the standard finite-length conclusion for descent methods
\citep{attouch2013convergence,wang2019global}:
\begin{equation}
\sum_{t\ge T}\|\mathcal{X}^{(t+1)}-\mathcal{X}^{(t)}\|<\infty.
\label{eq:finite-length}
\end{equation}
Hence $\{\mathcal{X}^{(t)}\}_{t\ge T}$ is Cauchy and converges to a
single limit. Projecting this convergence onto the variables
$(U,M,Y,\B{\alpha},\B{\beta},\B{\gamma})$ proves
Eq.~\eqref{eq:full-convergence}, and claim~(i) follows immediately
from Eq.~\eqref{eq:finite-length}.

For claim~(ii), each entry of
$\Delta M_{k}^{(t)}=M_{k}^{(t+1)}-M_{k}^{(t)}$ belongs to
$\{-2,0,2\}$. Therefore, whenever $\Delta M_{k}^{(t)}\ne0$,
$\|\Delta M_{k}^{(t)}\|_{F}\ge2$. The summability in claim~(i) implies
that only finitely many such nonzero jumps can occur. Thus the two
binary blocks are eventually constant; denote the common tail value by
$M^{\star}$.

It remains to verify stationarity. The dual update gives
$Y_{k}^{(t+1)}-Y_{k}^{(t)}=U_{k}^{(t+1)}-M_{k}^{(t+1)}$. Since the
left-hand side tends to zero by claim~(i), passing to the limit gives
$U_{k}^{\infty}=M_{k}^{\star}$. The binary update is the exact
entrywise projection
\begin{equation*}
M_{k}^{(t+1)}
\in
\argmin_{m\in\{\pm1\}^{n_{k}}}\|m-(U_{k}^{(t+1)}+Y_{k}^{(t)})\|_{F}^{2}.
\end{equation*}
Taking limits in this finite-dimensional inequality for every
$m\in\{\pm1\}^{n_{k}}$ yields the argmin condition in
Eq.~\eqref{eq:stationarity}. Finally, by claim~(ii) the term
$M_{k}^{(t+1)}-M_{k}^{(t)}$ in Eq.~\eqref{eq:grad-dual} is zero for
all sufficiently large $t$. Passing to the limit in
Eq.~\eqref{eq:grad-dual}, using the continuity of
$\nabla_{U_{k}}f$ on $\mathcal{C}$, gives the gradient--dual
stationarity equation in Eq.~\eqref{eq:stationarity}. If
$U_{k}^{\infty}+Y_{k}^{\infty}$ has no zero entries, the binary
projection decouples entrywise and is uniquely given by the sign of
that vector, proving the final assertion.
\end{proof}

\section{QAT-LoRDBA: Smooth-Sign STE}\label{app:qat-ste}

QAT-LoRDBA optimizes real-valued latent carrier matrices
$U_{1},U_{2}$ together with the scale vectors
$(\B{\alpha},\B{\beta},\B{\gamma})$. On the forward pass it exports the
binary carriers $B_{1}=\sign(U_{1})$ and $B_{2}=\sign(U_{2})$ inside
$\Delta W(\theta)$ as in Eq.~\eqref{eq:lordba}; on the backward pass
we use the \emph{smooth-sign} straight-through
estimator~\citep{leng2018extremely,liu2020reactnet} for the
non-differentiable sign maps:
\begin{equation}
\frac{\partial}{\partial u}\sign(u)
\;\longleftarrow\;
\frac{\partial}{\partial u}\tanh(\kappa u)
\;=\;\kappa\bigl(1-\tanh^{2}(\kappa u)\bigr),
\label{eq:smoothsign}
\end{equation}
with $\kappa = 100$ by default. The scales
$(\B{\alpha},\B{\beta},\B{\gamma})$ are trained in fp32 and do not require an
STE. We initialise the latent carriers $U_{1},U_{2}$ via the SVD warm start of
Eq.~\eqref{eq:svd-init} from a $10\%$-budget warm-up fp16 LoRA run,
which we find substantially stabilises optimisation at carrier rank
$R \ge 64$. Algorithm~\ref{alg:qat-lordba} summarises the training
loop.

\section{Training Mode Comparison}\label{app:training-modes}

Table~\ref{tab:ablation-variants} compares the four LoRDBA training modes at each size tier on \textsc{LLaMA-2-7B}. Several patterns emerge: (i)~QAT Full consistently leads, followed closely by QAT Freeze at matched BPW. (ii)~QAT Scratch, which does not require a pre-trained FP16 adapter, trails QAT Full by $4$--$9$\,pt on GSM8K, confirming that the PTQ warm-start is decisive for task-specific adapters. (iii)~PTQ-LoRDBA---training-free and requiring only ${\sim}2$\,min---degrades sharply at $\leq 2$\,BPW on math reasoning but retains reasonable summarization quality, achieving XSum $17.18$ vs.\ FP16 $18.87$ at 1\,BPW, revealing task-dependent quantization sensitivity that QAT closes.

\begin{table}[ht]
\centering
\caption{LLaMA-2-7B: LoRDBA training mode comparison (Avg $=$ (GSM8K + Minerva + XSum)/3).}
\label{tab:ablation-variants}
\small
\begin{tabular}{@{}l r r r@{}}
\toprule
\textbf{Mode} & \textbf{$\sim$7\,MB} & \textbf{$\sim$9.5\,MB} & \textbf{$\sim$19\,MB} \\
\midrule
QAT Full       & \textbf{24.14} & \textbf{25.59} & \textbf{28.08} \\
QAT Freeze     & --             & 25.61          & 27.85          \\
QAT Scratch    & --             & 21.46          & 24.13          \\
PTQ-LoRDBA     & 17.06          & 20.15          & 26.52          \\
\bottomrule
\end{tabular}
\vspace{2pt}
\parbox{\columnwidth}{\footnotesize
``--'' = configuration not evaluated at that tier.
QAT Freeze @2bpw matches QAT Full within $0.02$ Avg.
}
\end{table}

\paragraph{QAT Freeze training time.}
Despite training only scale vectors---$l(N{+}R{+}M)$ parameters per layer---QAT Freeze requires wall-clock time comparable to QAT Full, approximately $11$--$15$\,h at 2~epochs on H100, because the dominant cost is the full forward--backward pass through the frozen $4$-bit base model, which is identical across all QAT modes. The Adam optimizer state is smaller due to the absence of latent carriers, so peak memory is marginally lower, but each training step still performs the complete loss computation and gradient propagation through all Transformer layers. The parameter update itself is a negligible fraction of total step time.

\paragraph{BitDelta comparison.}
BitDelta~\citep{liu2024bitdelta} operates at $4.8$\,MB in its column$+$distill variant, a smaller size tier than LoRDBA's primary operating points of $7$--$19$\,MB. At this size, BitDelta achieves GSM8K $31.92$, Minerva $6.46$, XSum $18.70$, yielding Avg $19.03$, substantially below QAT Full @1bpw with Avg $24.14$ at $7.2$\,MB. The XSum result of $18.70$ is competitive, suggesting that for summarization at the ${\sim}5$\,MB tier a full-rank binary delta may be viable; however, the low-rank LoRDBA structure dominates at all tiers where both methods can operate.

\section{Extended Numerical Results}\label{app:extended-tables}

This appendix reports adapter-mode results that complement
Section~\ref{sec:exp-accuracy}: a comprehensive all-methods comparison
on \textsc{LLaMA-2-7B} across three tasks
(Table~\ref{tab:extended-7b}); a \textsc{LLaMA-3.2-3B} ablation
(Table~\ref{tab:extended-3b}); and the matched training-time
accounting (Table~\ref{tab:train_stats}). All
rows are evaluated under the adapter-mode protocol of
Appendix~\ref{app:setup}. Each task uses a \emph{task-specific}
LoRA FP16 $r{=}16$ adapter; all post-hoc methods compress the same
adapter per task.

\begin{table}[ht]
\centering
\caption{Llama-2-7B, rank 16: All LoRA adapter compression methods across three tasks.
Each task uses a \emph{task-specific} LoRA FP16 $r{=}16$ adapter.
GSM8K (8-shot exact match \%), Minerva Math (4-shot \%), XSum (ROUGE-L F1).
All post-hoc methods compress the same LoRA FP16 per task.}
\label{tab:extended-7b}
\resizebox{\textwidth}{!}{%
\begin{tabular}{llrrrrrr}
\toprule
 & & & & \multicolumn{2}{c}{\textbf{Math}} & \textbf{Summ} & \\
\cmidrule(lr){5-6} \cmidrule(lr){7-7}
\textbf{Method} & \textbf{Rank} & \textbf{BPW} & \textbf{Size (MB)} & \textbf{GSM8K} & \textbf{Minerva} & \textbf{XSum} & \textbf{Time$^\ddagger$} \\
\midrule
\multicolumn{8}{l}{\textit{Upper bound (full-precision LoRA)}} \\
LoRA FP16           & r16 & 16.0 &  76.3 & 54.38 & 13.17 & 18.87 & -- \\
LoRA FP16           & r64 & 16.0 & 305.0 & 58.68 & 15.52 &    -- & -- \\
\midrule
\multicolumn{8}{l}{\textit{Existing baselines (post-hoc, compressing LoRA FP16 r=16)}} \\
LoRAQuant $\rho$=0.5 & r16 & 1.20 &  5.7 & 36.09 &  7.18 & 17.67 & +36 min \\
LoRAQuant $\rho$=0.9 & r16 & 1.68 &  8.0 & 42.30 &  8.94 & 18.27 & +55 min \\
BitDelta-scalar      & r16 & 1.00 &  4.8 & 30.71 &  6.38 &    -- & +0 min \\
BitDelta-column      & r16 & 1.00 &  4.8 & 31.84 &  6.28 &    -- & +0 min \\
BitDelta-col+distill & r16 & 1.00 &  4.8 & 31.92 &  6.46 & \textbf{18.70} & +17 min \\
BitDelta-row         & r16 & 2.00 &  9.5 & 32.83 &  6.72 &    -- & +0 min \\
\midrule
\multicolumn{8}{l}{\textit{PTQ-LoRDBA (post-hoc, training-free, $\sim$2 min compression)}} \\
PTQ-LoRDBA          & r16 & 1.0 &  7.2 & 27.14 &  5.74 &  7.24 & +2 min \\
PTQ-LoRDBA          & r16 & 2.0 &  9.5 & 34.04 &  6.36 &    -- & +2 min \\
PTQ-LoRDBA          & r16 & 4.0 & 19.1 & 45.34 & 10.04 &    -- & +2 min \\
PTQ-LoRDBA          & r64 & 1.0 & 19.1 & 35.18 &  7.88 &    -- & +2 min \\
\midrule
\multicolumn{8}{l}{\textit{QAT Full (LoRA FP16 $\to$ PTQ $\to$ 1-epoch QAT refinement)}} \\
QAT Full            & r16 & 1.0 &  7.2 & 44.43 &  9.68 & 18.32 & +11h \\
QAT Full            & r16 & 2.0 &  9.5 & 48.60 &  9.68 & 18.49 & +11h \\
QAT Full            & r16 & 4.0 & 19.1 & 53.45 & \textbf{12.28} & 18.49 & +11h \\
QAT Full            & r64 & 1.0 & 19.1 & \textbf{53.60} & 12.04 & 18.60 & +15h \\
\midrule
\multicolumn{8}{l}{\textit{QAT Freeze (PTQ $\to$ freeze binary signs, train scaling vectors only)}} \\
QAT Freeze          & r16 & 1.0 &  7.2 & \textbf{44.20} & \textbf{8.18} &    -- & +11h \\
QAT Freeze          & r64 & 1.0 & 19.1 & \textbf{50.80} & \textbf{11.38} &    -- & +15h \\
\midrule
\multicolumn{8}{l}{\textit{QAT Scratch (end-to-end from random init, no FP16 required)}} \\
QAT Scratch         & r16 & 2.0 &  9.5 & 39.42 &  7.32 &    -- & 8h \\
QAT Scratch         & r16 & 4.0 & 19.1 & 45.94 &  8.50 &    -- & 9h \\
QAT Scratch         & r64 & 1.0 & 19.1 & 38.51 &  5.74 &    -- & 8h \\
\bottomrule
\end{tabular}%
}
\vspace{2pt}
\parbox{\textwidth}{\footnotesize
``--'' = not evaluated.
\textbf{Bold} = best among compressed methods at $\leq 19$\,MB at matching size tier.\\
$^\ddagger$Time = additional cost after LoRA FP16 training: Math $\sim$5h, Summ $\sim$3--6h on H100.
QAT Scratch does not require a pre-trained FP16 adapter. All runs use a single NVIDIA H100 80\,GB or B200.
}
\end{table}

\begin{table}[ht]
\centering
\caption{\textsc{LLaMA-3.2-3B}, rank 16: all LoRA adapter compression methods across three tasks.
Each task uses a \emph{task-specific} LoRA FP16 $r{=}16$ adapter.
GSM8K (8-shot exact match \%), Minerva Math (4-shot \%), XSum (ROUGE-L F1).
All post-hoc methods compress the same LoRA FP16 per task.}
\label{tab:extended-3b}
\resizebox{\textwidth}{!}{%
\begin{tabular}{llrrrrrr}
\toprule
 & & & & \multicolumn{2}{c}{\textbf{Math}} & \textbf{Summ} & \\
\cmidrule(lr){5-6} \cmidrule(lr){7-7}
\textbf{Method} & \textbf{Rank} & \textbf{BPW} & \textbf{Size (MB)} & \textbf{GSM8K} & \textbf{Minerva} & \textbf{XSum} & \textbf{Time$^\ddagger$} \\
\midrule
\multicolumn{8}{l}{\textit{Upper bound (full-precision LoRA)}} \\
LoRA FP16           & r16 & 16.0 &  46.4 & 59.36 & 17.56 & 15.38 & -- \\
LoRA FP16           & r64 & 16.0 & 185.5 & 61.56 & 20.60 &    -- & -- \\
\midrule
\multicolumn{8}{l}{\textit{Existing baselines (post-hoc, compressing LoRA FP16 r=16)}} \\
LoRAQuant $\rho$=0.5 & r16 & 1.22 &  3.6 & 45.26 & 14.10 & 15.02 & +8 min \\
LoRAQuant $\rho$=0.7 & r16 & 1.42 &  4.1 & 47.31 & 14.64 &    -- & +11 min \\
LoRAQuant $\rho$=0.9 & r16 & 1.71 &  5.0 & \textbf{51.63} & \textbf{15.20} & 15.49 & +17 min \\
BitDelta-scalar      & r16 & 1.00 &  2.9 & 42.46 & 13.00 &    -- & +0 min \\
BitDelta-column      & r16 & 1.00 &  2.9 & 42.68 & 12.66 &    -- & +0 min \\
BitDelta-col+distill & r16 & 1.00 &  2.9 & 42.61 & 12.68 & \textbf{16.78} & +15 min \\
BitDelta-row         & r16 & 2.00 &  5.8 & 43.29 & 13.18 &    -- & +0 min \\
\midrule
\multicolumn{8}{l}{\textit{PTQ-LoRDBA (post-hoc, training-free, $\sim$2 min compression)}} \\
PTQ-LoRDBA          & r16 & 1.5 &  4.4 & 41.32 & 12.20 & 15.52 & +2 min \\
PTQ-LoRDBA          & r16 & 2.0 &  5.8 & 48.22 & 12.86 & 15.55 & +2 min \\
PTQ-LoRDBA          & r16 & 4.0 & 11.6 & 55.04 & 16.50 & 15.19 & +2 min \\
PTQ-LoRDBA          & r64 & 1.0 & 11.6 & 49.81 & 14.60 &    -- & +2 min \\
\midrule
\multicolumn{8}{l}{\textit{QAT Full (LoRA FP16 $\to$ PTQ $\to$ 1-epoch QAT refinement)}} \\
\qcol QAT Full      & r16 & 1.5 &  4.4 & \textbf{52.77} & 13.86 & 15.34 & +7h \\
\qcol QAT Full      & r16 & 2.0 &  5.8 & 55.42 & \textbf{15.34} & 15.49 & +7h \\
\qcol QAT Full      & r16 & 4.0 & 11.6 & 56.63 & \textbf{17.44} & 15.35 & +7h \\
\qcol QAT Full      & r64 & 1.0 & 11.6 & 56.71 & 16.32 &    -- & +7h \\
\midrule
\multicolumn{8}{l}{\textit{QAT Freeze (PTQ $\to$ freeze binary signs, train scaling vectors only)}} \\
\qcol QAT Freeze    & r16 & 1.5 &  4.4 & 50.27          & 13.38         & 15.54 & +4h \\
\qcol QAT Freeze    & r64 & 1.0 & 11.6 & 57.16          & 16.46         &    -- & +4h \\
\qcol QAT Freeze    & r16 & 2.0 &  5.8 & \textbf{57.01} & 15.20         & 15.18 & +4h \\
\qcol QAT Freeze    & r16 & 4.0 & 11.6 & \textbf{58.98} & 17.10         & 15.38 & +4h \\
\midrule
\multicolumn{8}{l}{\textit{QAT Scratch (end-to-end from random init, no FP16 required)}} \\
\qcol QAT Scratch   & r16 & 1.5 &  4.4 & 45.56 & 12.42 & 14.72 & 7h \\
\qcol QAT Scratch   & r16 & 2.0 &  5.8 & 50.19 & 12.76 & 15.05 & 7h \\
\qcol QAT Scratch   & r16 & 4.0 & 11.6 & 54.66 & 14.14 & 15.07 & 7h \\
\qcol QAT Scratch   & r64 & 1.0 & 11.6 & 48.29 & 13.04 & 14.54 & 7h \\
\bottomrule
\end{tabular}%
}
\vspace{2pt}
\parbox{\textwidth}{\footnotesize
``--'' = not evaluated.
\textbf{Bold} = best among compressed methods at the corresponding size tier.\\
$^\ddagger$Time = additional cost after LoRA FP16 training, $\sim$5h on H100.
QAT Scratch does not require a pre-trained FP16 adapter. All runs use a single NVIDIA H100 80\,GB or B200.
}
\end{table}

\begin{table}[ht]
\centering
\caption{Adapter-mode QAT Full training statistics on
\textsc{LLaMA-2-7B} (MetaMathQA). LoRA pre-trained for $2$ epochs; QAT
for $1$ epoch (bf16 base, single NVIDIA B200).
``Peak'' is the device-side QAT training peak memory; ``ms/step'' is the
average wall-clock per QAT training step.}
\label{tab:train_stats}
\setlength{\tabcolsep}{5pt}
\renewcommand{\arraystretch}{0.95}
\resizebox{\textwidth}{!}{%
\begin{tabular}{@{}l l c c c c l@{}}
\toprule
Setting & Mode & Steps & ms/step & Time (min) & Peak (GB) & Hardware \\
\midrule
$r = 16$, LoRDBA $@1$BPW    & QAT Full & $24687$ & $804$  & $331$ & $55.6$ & B200 \\
$r = 16$, LoRDBA $@2$BPW    & QAT Full & $24687$ & $856$  & $352$ & $56.1$ & B200 \\
$r = 16$, LoRDBA $@4$BPW    & QAT Full & $24687$ & $899$  & $370$ & $58.0$ & B200 \\
$r = 64$, LoRDBA $@1$BPW    & QAT Full & $24687$ & $913$  & $376$ & $58.0$ & B200 \\
\bottomrule
  \end{tabular}%
  }
\end{table}

\paragraph{Reconstruction error.}
Table~\ref{tab:recon} records the ADMM-converged relative
reconstruction error of $\Delta W^{\star}$ by PTQ-LoRDBA as a function
of BPW. The monotonic decay with rank is qualitatively consistent with
Theorem~\ref{thm:expressivity} and with the empirical
residual-to-mean ratios of Table~\ref{tab:ratio}.

\begin{table}[ht]
\centering
\caption{ADMM-converged per-element relative reconstruction error
$\|\Delta W^{\star} - \Delta W(\theta)\|_{F}/\|\Delta W^{\star}\|_{F}$
on \textsc{LLaMA-2-7B} (averaged across projections).}
\label{tab:recon}
\small
\begin{tabular}{l c c c c c}
\toprule
BPW & 0.5 & 1.0 & 2.0 & 4.0 & 8.0 \\
\midrule
effective carrier rank $R$ & $\sim 24$ & $\sim 47$ & $\sim 94$ & $\sim 190$ & $\sim 380$ \\
Relative error ratio & 0.358      & 0.198      & 0.099      & 0.040       & 0.015       \\
\bottomrule
\end{tabular}
\end{table}

\section{Pseudocode}\label{app:pseudocode}

\subsection{Closed-form scale updates for PTQ-LoRDBA}
\label{app:scale-updates}

This subsection derives the per-axis least-squares updates used by
Eq.~\eqref{eq:scales-step}. Throughout, let $R$ denote the LoRDBA
binary carrier rank of Definition~\ref{def:lordba} (so that
$B_{1} \in \{\pm 1\}^{N\times R}$ and $B_{2} \in \{\pm 1\}^{R\times M}$);
to avoid the notation clash, we use the symbol $T$ for the regression
target, defined below. We treat the canonical single-envelope case
$\ell=1$; the multi-envelope case applies the same updates
block-coordinate-wise to each envelope while holding the remaining
envelopes fixed (see the residual definition at the end of this
subsection).

\paragraph{Setup.}
Fix the binary carriers $B_{1},B_{2}$ and define the regression target
\begin{equation}
T\;\coloneqq\;\Delta W^{\star}\;\in\;\mab{R}^{N\times M}.
\label{eq:scale-target}
\end{equation}
The PTQ-LoRDBA scale subproblem is the joint minimisation
\begin{equation}
\min_{\B{\alpha}\in\mab{R}^{N}, \B{\beta}\in\mab{R}^{R}, \B{\gamma}\in\mab{R}^{M}}
\;\frac{1}{2} \bigl\|T-\diag(\B{\alpha})B_{1}\diag(\B{\beta})B_{2}\diag(\B{\gamma})\bigr\|_{F}^{2},
\label{eq:scale-joint}
\end{equation}
which we solve by per-axis block-coordinate descent. Each per-axis
subproblem admits a closed-form least-squares solution, derived next.

\paragraph{$\B{\alpha}$-update (rowwise).}
For fixed $B_{1},B_{2},\B{\beta},\B{\gamma}$, set
\begin{equation}
D_{\alpha}\;\coloneqq\;B_{1}\diag(\B{\beta})B_{2}\diag(\B{\gamma})\;\in\;\mab{R}^{N\times M}.
\label{eq:alpha-D}
\end{equation}
Substituting Eq.~\eqref{eq:alpha-D} into Eq.~\eqref{eq:scale-joint} and
expanding the Frobenius norm rowwise,
\begin{equation*}
\frac{1}{2}\bigl\|T-\diag(\B{\alpha})D_{\alpha}\bigr\|_{F}^{2}
\;=\;\sum_{i=1}^{N}\frac{1}{2}\bigl\|T_{i,:}-\alpha_{i} (D_{\alpha})_{i,:}\bigr\|_{2}^{2},
\end{equation*}
which decouples over the row index $i$. Each scalar problem
\begin{equation}
\min_{\alpha_{i}\in\mab{R}}\;\frac{1}{2} \bigl\|T_{i,:}-\alpha_{i}(D_{\alpha})_{i,:}\bigr\|_{2}^{2}
\label{eq:alpha-row}
\end{equation}
has the unique minimum-norm least-squares solution
\begin{equation}
\alpha_{i}
\;=\;
\begin{cases}
\dfrac{\langle T_{i,:},(D_{\alpha})_{i,:}\rangle}{\|(D_{\alpha})_{i,:}\|_{2}^{2}},
& \|(D_{\alpha})_{i,:}\|_{2}>0,\\[0.8em]
0,
& \|(D_{\alpha})_{i,:}\|_{2}=0,
\end{cases}
\qquad i \in [N].
\label{eq:alpha-update}
\end{equation}

\paragraph{$\B{\gamma}$-update (columnwise).}
For fixed $B_{1},B_{2},\B{\alpha},\B{\beta}$, set
\begin{equation}
D_{\gamma}\;\coloneqq\;\diag(\B{\alpha})B_{1}\diag(\B{\beta})B_{2}\;\in\;\mab{R}^{N\times M}.
\label{eq:gamma-D}
\end{equation}
The objective Eq.~\eqref{eq:scale-joint} expands columnwise as
\begin{equation*}
\frac{1}{2}\bigl\|T-D_{\gamma}\diag(\B{\gamma})\bigr\|_{F}^{2}
\;=\;\sum_{j=1}^{M}\frac{1}{2}\bigl\|T_{:,j}-\gamma_{j}(D_{\gamma})_{:,j}\bigr\|_{2}^{2},
\end{equation*}
and analogously decouples over $j \in [M]$. The unique minimum-norm
least-squares solution is
\begin{equation}
\gamma_{j}
\;=\;
\begin{cases}
\dfrac{\langle T_{:,j},(D_{\gamma})_{:,j}\rangle}{\|(D_{\gamma})_{:,j}\|_{2}^{2}},
& \|(D_{\gamma})_{:,j}\|_{2}>0,\\[0.8em]
0,
& \|(D_{\gamma})_{:,j}\|_{2}=0,
\end{cases}
\qquad j \in [M].
\label{eq:gamma-update}
\end{equation}

\paragraph{$\B{\beta}$-update (R-dimensional).}
The $\B{\beta}$-update is an $R$-dimensional least-squares problem.
For fixed $B_{1},B_{2},\B{\alpha},\B{\gamma}$, define
\begin{equation}
P\;\coloneqq\;\diag(\B{\alpha})B_{1}\;\in\;\mab{R}^{N\times R},
\qquad
Q\;\coloneqq\;B_{2}\diag(\B{\gamma})\;\in\;\mab{R}^{R\times M}.
\label{eq:beta-PQ}
\end{equation}
Expanding the bilinear identity
\begin{equation}
\diag(\B{\alpha})B_{1}\diag(\B{\beta})B_{2}\diag(\B{\gamma})
\;=\;P\diag(\B{\beta})Q
\;=\;\sum_{k=1}^{R}\beta_{k} P_{:,k}Q_{k,:},
\label{eq:beta-bilinear}
\end{equation}
where the second equality uses the column--row decomposition
$\diag(\B{\beta}) = \sum_{k}\beta_{k}\B{e}_{k}\B{e}_{k}^{\top}$ and
$P\B{e}_{k}\B{e}_{k}^{\top}Q = P_{:,k}Q_{k,:}$. Vectorising
Eq.~\eqref{eq:beta-bilinear} gives
$\operatorname{vec}(P\diag(\B{\beta})Q) = H\B{\beta}$ with
$H \in \mab{R}^{NM\times R}$ defined by
\begin{equation}
H_{:,k}\;\coloneqq\;\operatorname{vec}(P_{:,k}Q_{k,:}),\qquad k \in [R].
\label{eq:beta-H-cols}
\end{equation}
Hence Eq.~\eqref{eq:scale-joint} reduces, in the $\B{\beta}$-direction,
to the standard linear least-squares problem
\begin{equation}
\min_{\B{\beta}\in\mab{R}^{R}}\;\frac{1}{2} \bigl\|\operatorname{vec}(T)-H\B{\beta}\bigr\|_{2}^{2},
\label{eq:beta-ls}
\end{equation}
with normal equations $H^{\top}H\B{\beta} = H^{\top}\operatorname{vec}(T)$. We
record these in our notation as
\begin{equation}
G_{\beta} \B{\beta}\;=\;\B{h}_{\beta},
\label{eq:beta-normal}
\end{equation}
where, for $k,l \in [R]$,
\begin{equation}
(G_{\beta})_{kl}
\;=\;\langle P_{:,k},P_{:,l}\rangle \langle Q_{k,:},Q_{l,:}\rangle,
\qquad
(\B{h}_{\beta})_{k}
\;=\;\langle T,P_{:,k}Q_{k,:}\rangle_{F}.
\label{eq:beta-gram}
\end{equation}
Equivalently, in matrix form,
\begin{equation}
G_{\beta}\;=\;(P^{\top}P)\odot(QQ^{\top}),
\label{eq:beta-Schur}
\end{equation}
where $\odot$ is the Hadamard product. Since $P^{\top}P$ and $QQ^{\top}$
are positive semidefinite, $G_{\beta}$ is positive semidefinite by the
Schur product theorem. When $G_{\beta}$ is nonsingular, the update is
$\B{\beta} = G_{\beta}^{-1}\B{h}_{\beta}$; when singular, we use the
unique minimum-norm least-squares solution
\begin{equation}
\B{\beta}\;=\;G_{\beta}^{\dagger}\B{h}_{\beta},
\label{eq:beta-pinv}
\end{equation}
where $G_{\beta}^{\dagger}$ denotes the Moore--Penrose pseudoinverse.

\paragraph{Generalisation to envelope rank $\ell \ge 2$.}
For envelope rank $\ell \ge 2$, the same per-axis updates
Eqs.~\eqref{eq:alpha-update}, \eqref{eq:gamma-update},
and \eqref{eq:beta-pinv} apply to one envelope $i \in [\ell]$ at a
time after replacing the target $T$ by the per-envelope residual
\begin{equation}
T^{(i)}
\;\coloneqq\;
\Delta W^{\star}
- \sum_{m\ne i} 
\diag(\B{\alpha}^{(m)})B_{1}\diag(\B{\beta}^{(m)})B_{2}\diag(\B{\gamma}^{(m)}),
\qquad i \in [\ell].
\label{eq:per-envelope-target}
\end{equation}
Cycling the envelope index $i \in [\ell]$ in any order produces a
block-coordinate descent on Eq.~\eqref{eq:scale-joint} extended to
$\ell$ envelopes, and each per-envelope update inherits the closed
form above with $T$ replaced by $T^{(i)}$.
\begin{algorithm}[ht]
\caption{PTQ-LoRDBA at $\ell = 1$ (scaled consensus ADMM).}
\label{alg:ptq-lordba}
\begin{algorithmic}[1]
\Require fp16 LoRA factors $A^{\star},B^{\star}$, carrier rank $R$, outer iterations $K$, penalty schedule $\{\rho^{(t)}\}$
\Ensure LoRDBA adapter $\theta^{\star} = (B_{1}^{\star},B_{2}^{\star},\B{\alpha}^{\star},\B{\beta}^{\star},\B{\gamma}^{\star})$
\State $(U_{\Sigma},S_{\Sigma},V_{\Sigma})\gets\mathrm{thinSVD}(A^{\star}(B^{\star})^{\top},\text{rank} = R)$ \Comment{Eq.~\eqref{eq:svd-init}}
\State $U_{1}^{(0)} \gets \sign(U_{\Sigma})$, $U_{2}^{(0)} \gets \sign(V_{\Sigma}^{\top})$, $\B{\beta}^{(0)} \gets \diag(S_{\Sigma})$
\State $M_{k}^{(0)} \gets U_{k}^{(0)}$, $Y_{k}^{(0)} \gets 0$ \quad for $k \in \{1,2\}$
\State Initialise $(\B{\alpha}^{(0)},\B{\gamma}^{(0)})$ by one block-coordinate sweep of Eq.~\eqref{eq:scales-step} at $(U^{(0)},\B{\beta}^{(0)})$
\For{$t=0,1,\dots,K - 1$}
  \State Update $U_{1}^{(t+1)}$, $U_{2}^{(t+1)}$ by Tikhonov least-squares (Eqs.~\eqref{eq:u-step-1}--\eqref{eq:u-step-2})
  \State Update $(\B{\alpha}^{(t+1)},\B{\beta}^{(t+1)},\B{\gamma}^{(t+1)})$ by one block-coordinate sweep of Eq.~\eqref{eq:scales-step} (per-axis closed-form least-squares)
  \State $M_{k}^{(t+1)} \gets \sign(U_{k}^{(t+1)} + Y_{k}^{(t)})$ for $k \in \{1,2\}$
  \State $Y_{k}^{(t+1)} \gets Y_{k}^{(t)} + U_{k}^{(t+1)} - M_{k}^{(t+1)}$ for $k \in \{1,2\}$
  \State Residual-balance $\rho^{(t+1)}$ (Boyd, truncated at $K/2$)
  \If{$M^{(t+1)} = M^{(t)}$} \textbf{break}  \Comment{discrete freeze; Theorem~\ref{thm:convergence}(ii)}
  \EndIf
\EndFor
\State \Return $(M_{1}^{(t)},M_{2}^{(t)},\B{\alpha}^{(t)},\B{\beta}^{(t)},\B{\gamma}^{(t)})$
\end{algorithmic}
\end{algorithm}

\begin{algorithm}[ht]
\caption{QAT-LoRDBA (smooth-sign STE).}
\label{alg:qat-lordba}
\begin{algorithmic}[1]
\Require base model $W_{0}$, stream $\{(\B{x}_{t},\B{y}_{t})\}$, carrier rank $R$, STE temperature $\kappa$, optimiser $\mathsf{opt}$
\State Initialise $(U_{1},U_{2})$ via Eq.~\eqref{eq:svd-init} from a $10\%$-budget fp16 LoRA warm-up
\State Initialise scales $(\B{\alpha},\B{\beta},\B{\gamma})$ by one block-coordinate sweep of Eq.~\eqref{eq:scales-step} (per-axis closed-form least-squares)
\For{each training step $t$}
  \State Forward: $\Delta W \gets \diag(\B{\alpha})\sign(U_{1})\diag(\B{\beta})\sign(U_{2})\diag(\B{\gamma})$
  \State Compute task loss $\ell_{t}(W_{0}+\Delta W;\B{x}_{t},\B{y}_{t})$
  \State Backward: replace $\partial\sign(U)/\partial U$ by $\kappa(1-\tanh^{2}(\kappa U))$ in autograd
  \State Step $(U_{1},U_{2},\B{\alpha},\B{\beta},\B{\gamma}) \gets \mathsf{opt}(\cdot,\nabla\ell_{t})$
\EndFor
\State \Return $(B_{1},B_{2},\B{\alpha},\B{\beta},\B{\gamma}) \gets (\sign(U_{1}),\sign(U_{2}),\B{\alpha},\B{\beta},\B{\gamma})$
\end{algorithmic}
\end{algorithm}

\begin{algorithm}[ht]
\caption{PTQ-LoRDBA at general envelope rank $\ell \ge 1$.}
\label{alg:ptq-lordba-general}
\begin{algorithmic}[1]
\Require fp16 LoRA factors $A^{\star},B^{\star}$, carrier rank $R$, envelope rank $\ell \ge 1$, outer iterations $K$
\State $(U_{\Sigma},S_{\Sigma},V_{\Sigma})\gets\mathrm{thinSVD}(A^{\star}(B^{\star})^{\top},\text{rank} = R)$
\State Split singular spectrum into $\ell$ disjoint index sets $\mathcal{I}_{1},\dots,\mathcal{I}_{\ell} \subseteq [R]$; initialise $\B{\beta}^{(i)}$ from $\{(S_{\Sigma})_{jj}\}_{j\in\mathcal{I}_{i}}$ (extended by zeros to $\mab{R}^{R}$)
\State Initialise $(B_{1},B_{2})$ by Eq.~\eqref{eq:svd-init}, and $\{(\B{\alpha}^{(i)},\B{\gamma}^{(i)})\}_{i=1}^{\ell}$ by solving the per-envelope scale step
\For{$t=0,1,\dots,K - 1$}
  \State $(U_{1},U_{2})$-update: outer-loop sign-ADMM step (Algorithm~\ref{alg:ptq-lordba}, lines~6--10) with $f$ replaced by the envelope-sum objective $\frac{1}{2}\|\Delta W^{\star} - \sum_{i}\diag(\B{\alpha}^{(i)})U_{1}\diag(\B{\beta}^{(i)})U_{2}\diag(\B{\gamma}^{(i)})\|_{F}^{2}$
  \State Update $\{(\B{\alpha}^{(i)},\B{\beta}^{(i)},\B{\gamma}^{(i)})\}_{i=1}^{\ell}$ by block-coordinate closed-form least-squares, cycling over envelope index $i$
\EndFor
\State \Return $\bigl(B_{1},B_{2},\{(\B{\alpha}^{(i)},\B{\beta}^{(i)},\B{\gamma}^{(i)})\}_{i=1}^{\ell}\bigr)$
\end{algorithmic}
\end{algorithm}

\section{Inference-Kernel Details}\label{app:kernel}

Equation~\eqref{eq:lordba-fwd} expresses the LoRDBA forward pass as
three element-wise scale fusions and two binary matmuls. Each binary
matmul $\B{z} = B^{\top}\B{w}$ with $B \in \{\pm 1\}^{p\times q}$ and
fp16 input $\B{w} \in \mab{R}^{p}$ computes, for every $j \in [q]$,
the exact sign-accumulation
\begin{equation}
z_{j}\;=\;\sum_{i=1}^{p}B_{ij}w_{i}
\;=\;\sum_{i : B_{ij}=+1} w_{i}\;-\;\sum_{i : B_{ij}=-1} w_{i},
\label{eq:sign-accum-app}
\end{equation}
i.e.\ a signed sum of fp16 values with no floating-point multiplications.
On contemporary GPUs this primitive is realised by the
\texttt{gemlite} library~\citep{badri2023gemlite} using warp-level
predicated fp16 reductions (the same kernel used by
DBF~\citep{boza2026dbf}, BitNet~\citep{wang2023bitnet}, and
OneBit~\citep{xu2024onebit}), giving a measured $\approx 2\times$
speedup over an INT4 GEMM at typical adapter shapes.
Table~\ref{tab:kernel-flops} summarises the arithmetic intensity.

\begin{table}[ht]
\centering
\caption{Arithmetic-intensity breakdown of the adapter matmul on a
host projection of shape $N \times M$. The fp16/INT4 LoRA baselines
use rank $r_{0}$; LoRDBA uses binary carrier rank $R$ and
envelope rank $\ell$ (the table is for $\ell = 1$). One fp16
multiplication and one fp16 addition count as two separate ops.
LoRDBA trades two fp16 GEMMs for two sign-accumulation matmuls (zero
floating-point multiplications on the inner edges) plus three
channel-wise vector scales that fuse into the residual-stream epilogue
with the scale traffic already accounted for in the bytes column.}
\label{tab:kernel-flops}
\resizebox{\textwidth}{!}{%
\begin{tabular}{l c c c}
\toprule
Method & Bytes loaded (per output row) & Arithmetic (ops/row) & Dominant kernel \\
\midrule
LoRA fp16                & $2r_{0}(N{+}M)$             & $2r_{0}(N{+}M)$, half muls / half adds              & fp16 GEMM \\
LoftQ (INT4)             & $0.5r_{0}(N{+}M)$           & $2r_{0}(N{+}M)$, half muls / half adds              & INT4 GEMM \\
BitDelta (full-$\Delta$) & $NM/8 + O(1)$               & $NM$ fp16 adds                                      & addition-only GEMM \\
LoRDBA ($\ell = 1$)    & $R(N{+}M)/8 + 2(N{+}R{+}M)$ & $R(N{+}M)$ fp16 adds $+O(N{+}R{+}M)$ muls           & sign-accumulation GEMM \\
\bottomrule
\end{tabular}%
}
\end{table}

\clearpage
\section{Implementation Details and Compute}\label{app:setup}
\label{app:compute}

\paragraph{Hyper-parameters.}
\textbf{PTQ-LoRDBA.} $K = 100$ outer iterations; Boyd
residual-balancing with $(\tau,\mu) = (2,10)$; scale-matched warm
penalty $\rho^{(0)} = \|\Delta W^{\star}\|_{F}^{2}/(NR + RM)$. The
SVD warm start of Eq.~\eqref{eq:svd-init} is used throughout.
\textbf{QAT-LoRDBA.} LoRA warm-up: two epochs, AdamW, $\mathrm{lr} = 2 \times 10^{-4}$;
QAT refinement: one epoch, $\mathrm{lr} = 5 \times 10^{-5}$.
Both use linear warm-up over $5\%$ of steps, cosine to zero, $\kappa = 100$.
\textbf{Envelope rank.} Identical schedule for $\ell \ge 2$, with
the split-spectrum initialisation of
Algorithm~\ref{alg:ptq-lordba-general}.

\paragraph{Compute budget.}
On a single NVIDIA H100 (80\,GB), PTQ-LoRDBA completes in ${\sim}2$
minutes per model, while QAT Full requires ${\sim}11$h (1 epochs)
for \textsc{LLaMA-2-7B} on MetaMathQA.
Full-run numbers are in Table~\ref{tab:compute-app}.

\begin{table}[ht]
\centering
\caption{Wall-clock budget for adapter production on a single
NVIDIA H100 (80GB).}
\label{tab:compute-app}
\small
\begin{tabular}{l c c}
\toprule
Stage                       & LLaMA-2-7B & LLaMA-3.2-3B \\
\midrule
fp16 LoRA fine-tuning       &  $\sim$5 h    & $\sim$5 h \\
PTQ-LoRDBA ($\ell = 1$)    &  $\sim$2 min  & $\sim$2 min \\
QAT Full ($\ell = 1$, 2 ep) &  $\sim$11 h   & $\sim$7 h \\
QAT Freeze ($\ell = 1$, 2 ep) &  $\sim$11--15 h & $\sim$7 h \\
QAT Scratch ($\ell = 1$, 2 ep) & $\sim$8--9 h & $\sim$7 h \\
\bottomrule
\end{tabular}
\end{table}

\section{Additional Figures}\label{app:additional-figs}

\begin{figure}[ht]
  \centering
  \includegraphics[width=0.98\linewidth]{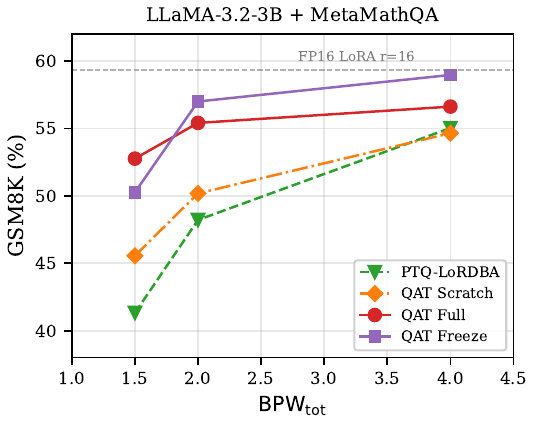}
  \caption{BPW sweep on \textsc{LLaMA-3.2-3B}+\textsc{MetaMathQA} (GSM8K, $r{=}16$).
  Four training modes are compared; PTQ saturates below $2$ BPW while QAT variants close the gap to FP16 LoRA.}
  \label{fig:ablation}
\end{figure}

\stopcontents[appendix]

\end{document}